\documentclass[twoside]{article}

\usepackage[preprint]{aistats2027}
\usepackage[round,authoryear]{natbib}

\usepackage{hyperref}         
\hypersetup{
  colorlinks=true,            
  linkcolor=blue,             
  urlcolor=purple,            
  citecolor=teal              
}

\usepackage[english]{babel}
\usepackage[colorinlistoftodos,prependcaption,textsize=tiny]{todonotes}

\usepackage{amssymb, amsmath, mathtools}  
\usepackage{amsthm}           
\usepackage{bbold}            
\usepackage{mathrsfs}

\usepackage{graphicx}         
\usepackage{caption}          
\usepackage{subcaption}       
\usepackage{xspace}
\usepackage{xcolor}
\usepackage{graphicx}         
\usepackage{ulem}
\usepackage{colortbl}

\newcommand{\sub}{\subset}                  
\newcommand{\R}{\mathbb{R}}                 
\newcommand{\N}{\mathbb{N}}                 
\newcommand{\Expval}{\mathbb{E}}            
\newcommand{\Prob}{\mathbb{P}}              
\newcommand{\posterior}{Q}                 
\newcommand{\prior}{P}                      
\newcommand{\indicatorfunc}[1]{\mathbb{1}_{\{#1\}}}       
\newcommand{\calN}{\mathcal{N}}             
\newcommand{\inspace}{\mathcal{X}}          
\newcommand{\hypoclass}{\mathcal{H}}        
\newcommand{\losshypoclass}{\mathcal{G}}        
\newcommand{\measfunc}[1]{\mathcal{M}(#1)}         
\newcommand{\bounmeasfunc}[1]{\mathcal{M}_b(#1)}        
\newcommand{\spaceS}{\mathcal{S}}           
\newcommand{\spaceT}{\mathcal{T}}           
\newcommand{\samplespace}{\mathcal{Z}}      
\newcommand{\borel}[1]{\mathscr{B}(#1)}     
\newcommand{\sigHypo}{\mathscr{H}}          
\newcommand{\sigF}{\mathscr{F}}             
\newcommand{\sigS}{\mathscr{S}}             
\newcommand{\sigZ}{\mathscr{Z}}             
\newcommand{\loss}{\ell}                    
\newcommand{\risk}{\mathcal{R}_\loss}             
\newcommand{\emprisk}{\hat{\mathcal{R}}_\loss}    
\newcommand{\RVX}{X}                        
\newcommand{\RVY}{Y}                        
\newcommand{\RVZ}{Z}                        
\newcommand{\samplesize}{n}                 
\newcommand{\RVsample}{{S_\samplesize}}     
\newcommand{\sample}{{s_\samplesize}}       
\newcommand{\idmat}{I}                      
\newcommand{\KL}[2]{D_{\operatorname{KL}}(#1 \| #2)}  
\newcommand{\entropfunc}{f}              
\newcommand{\costfunc}{c}                   
\newcommand{\opttransdist}[2]{\mathrm{OT}_\costfunc(#1 , #2)}  
\newcommand{\wassersteinpdist}[2]{W_p(#1 , #2)}     
\newcommand{\TVdist}[2]{d_{\mathrm{TV}}(#1 , #2)}   
\newcommand{\phiOT}[2]{\mathsf{C}_{\phi}(#1 , #2)}  
\newcommand{\gengap}{\Delta_\samplesize}    
\newcommand{\constgengap}{C_\samplesize}    
\newcommand{\rademachercompl}{\mathfrak{R}}
\newcommand{\VCdim}{\operatorname{VC}}
\newcommand{\epscovnumb}{\mathcal{N}}
\newcommand{\map}[3]{#1\colon #2 \to #3}    

\theoremstyle{definition}
\newtheorem{theorem}{Theorem}[section]          
\newtheorem{definition}[theorem]{Definition}    
\newtheorem{lemma}[theorem]{Lemma}              
\newtheorem{proposition}[theorem]{Proposition}  
\newtheorem{remark}[theorem]{Remark}            
\newtheorem{assumption}[theorem]{Assumption}    

\begin{document}

\twocolumn[

\aistatstitle{Optimal Transport PAC-Bayes Bounds}

\aistatsauthor{ Armin Beck$^{1,2}$ \And Peter Ochs$^2$ \And Jalal Fadili$^3$ }

\aistatsaddress{ $^1$ Max Planck Institute for Informatics, Saarland Informatics Campus, Saarbrücken, Germany \\ 
    $^2$Saarland University, Saarland Informatics Campus, Saarbrücken, Germany \\
    $^3$ENSICAEN, CNRS GREYC, Caen, France} ]
\begingroup
\renewcommand{\thefootnote}{}
\footnotetext{Correspondence to: \href{mailto:armin.beck@math.uni-sb.de}{armin.beck@math.uni-sb.de}}
\endgroup

\begin{abstract}
    PAC-Bayesian generalization bounds are among the strongest distribution-dependent guarantees in statistical learning and are typically expressed in terms of the Kullback-Leibler (KL) divergence between a prior and posterior distribution. While analytically convenient, the KL divergence requires the posterior to be absolutely continuous w.r.t.\ the prior, imposing support constraints that can substantially restrict the choice of posterior. Moreover, the KL divergence does not account for the underlying geometry of the probability space. In this work, we derive new PAC-Bayesian bounds replacing the KL divergence with general optimal transport (OT) distances, including the total variation and Wasserstein-$p$ distances.
\end{abstract}

\section{Introduction}
\label{section: introduction}

The ability of a learning algorithm to perform accurately on unseen data is fundamental for real-world applications, ranging from medical diagnosis to autonomous systems. Classical statistical learning theory provides tools to analyze this ability, most notably the Vapnik–Chervonenkis (VC) dimension \citep{vapnik1995nature} and Rademacher complexity \citep{bartlett2002rademacher}, which bound the generalization error by quantifying the capacity of a hypothesis class. However, classical generalization bounds are often overly pessimistic and distribution-agnostic. They provide uniform guarantees over the entire hypothesis class by considering worst-case scenarios and ignoring the structure of the underlying data distribution. In particular, in the modern over-parameterized regime of deep neural networks, where models generalize well in practice, these classical bounds become vacuous, since the hypothesis classes involved are extremely large.

These limitations have motivated the development of the Probably Approximately Correct Bayes (PAC-Bayes) framework \citep{mcallester1999pacbayesian}. Unlike classical approaches, PAC-Bayes theory does not aim to provide uniform guarantees over the whole hypothesis class. Instead, it yields data-dependent bounds by considering probability distributions over hypotheses. More precisely, classical PAC-Bayes bounds depend on a measure of discrepancy between a prior distribution and a posterior distribution over hypotheses, which is most commonly quantified using the Kullback–Leibler (KL) divergence. However, the KL divergence can be very large even when the prior and posterior are close w.r.t.\ other probability distances that account for the geometry of the hypothesis space, which may again lead to loose generalization bounds.
Moreover, the KL divergence requires the posterior to be absolutely continuous w.r.t.\ the prior, implying that the posterior must be supported on a subset of the prior’s support. As a consequence, the choice of the prior becomes crucial. If the prior assigns mass predominantly to poorly performing hypotheses, the posterior cannot concentrate on high-performing hypotheses.

Recent work has proposed PAC-Bayes bounds that replace the KL divergence. One approach uses general $\entropfunc$-divergences, which still require the posterior to be absolutely continuous w.r.t.\ the prior and therefore inherits similar support restrictions. A second line of work instead uses integral probability metrics (IPMs), which relax this absolute continuity requirement but require the geometry to be introduced implicitly through the choice of a test function class. In this paper, we derive a new class of PAC-Bayesian bounds based on optimal transport (OT) distances with arbitrary ground costs, hence covering a major class of statistical distances between probability measures. Unlike $\entropfunc$-divergence-based bounds, our OT-based bounds do not suffer from support restrictions and unlike most IPM-based bounds, they are sensitive to the underlying geometry of the hypothesis space, routinely yielding tighter and non-vacuous generalization bounds.

Our main contributions are:
\begin{itemize}
    \item We derive a general PAC-Bayesian generalization bound that replaces the classical KL divergence between the prior and posterior by a general OT distance, thereby allowing the discrepancy measure to incorporate the geometry of the hypothesis space.
    \item We specialize the framework to several relevant choices of transport cost, including total variation (TV) and Wasserstein-$p$ distances, while TV is also a special case of IPMs, the Wasserstein-$p$ has not been considered before.
\end{itemize}

\subsection{Related Work}

The PAC-Bayes framework provides a powerful and flexible approach for deriving generalization guarantees for randomized predictors. Early foundational work by \cite{mcallester1999pacbayesian} established basic PAC-Bayesian inequalities, which were later refined by \cite{catoni2007pacbayes} to obtain sharper and more robust bounds for supervised learning. A variety of classical PAC-Bayes bounds have since been developed, including well-known results by \cite{seeger2003pacbayes}, \cite{maurer2004note}, \cite{langford2002pacbayes}, \cite{germain2009pacbayes} and \cite{mcallester2003simplified}, which provide explicit and practically usable generalization guarantees. More recently, PAC-Bayes theory has been successfully applied to modern deep learning models, yielding non-vacuous bounds for neural networks, notably in the works of \cite{dziugaite2017computing}, \cite{arora2018strongercompression} and \cite{zhou2019nonvacuous}. For a modern and accessible introduction to PAC-Bayes theory, we refer to the survey by \cite{alquier2024pacbayes}. A common feature of essentially all of these results is their reliance on the KL divergence to quantify the discrepancy between the prior and posterior distributions over hypotheses.

Motivated by the limitations of KL-based bounds, several works have explored PAC-Bayes generalization bounds based on alternative discrepancy measures. \cite{begin2016renyi} derived PAC-Bayes bounds in which the KL divergence is replaced by R\'enyi divergences, extending the classical framework while retaining a divergence-based structure. \cite{alquier2017hostile} further generalized this approach by replacing the KL divergence with a general Csisz\'ar $\entropfunc$-divergence. While these extensions provide additional flexibility, both R\'enyi divergences and $\entropfunc$-divergences still require absolute continuity of the posterior w.r.t.\ the prior and thus inherit many of the limitations of KL-based bounds. In the approach by \cite{amit2022integral}, they introduced PAC-Bayes-style bounds based on IPMs, including TV distance and the Wasserstein-1 distance. Similarly, \cite{viallard2023wasserstein} derived PAC-Bayes bounds based on the Wasserstein-1 distance and proposed PAC-Bayesian learning algorithms that optimize Wasserstein-based objectives instead of KL divergence.

OT theory has recently attracted increasing attention in machine learning due to its ability to capture geometric structure in probability spaces \citep{Peyr2018ComputationalOT}. See also \citep{villani2021topics} for a comprehensive mathematical treatment. In the context of learning theory, OT distances have been used to derive generalization bounds in different settings. \cite{chuang2021measuringgeneralization} developed margin-based generalization bounds in which margins are normalized by OT costs, while \cite{zhang2018optimaltransportgeneralization} introduced the notion of algorithmic transport cost, defined as the expected Wasserstein-1 distance between prior and posterior distributions, to control generalization error.

Notably, across both PAC-Bayesian and non-PAC-Bayesian analyses, the Wasserstein-1 distance is the only OT distance that has so far been used to obtain provable generalization guarantees. The proofs strongly rely on the equivalent characterization of Wasserstein-1 as an IPM. In contrast, this work provides a general PAC-Bayesian framework that accommodates arbitrary OT distances including OT distances with concave costs, the TV distance and Wasserstein-$p$ distances. Extending beyond the Wasserstein-1 case requires quite different arguments from the one based on its IPM characterization. In the standard KL based PAC-Bayesian analysis, the Donsker-Varadhan variational formula connects the KL divergence to exponential moments of the generalization gap. For general OT distances, we instead exploit Kantorovich duality, which requires controlling the generalization gap w.r.t.\ the geometry induced by the transport cost. This leads to a regularity condition that requires the chosen transport cost to define a geometry that is sufficiently compatible with the learning problem. This condition is not automatically satisfied and must be verified for the learning problem and cost function at hand. We show how such verification can be carried out using classical learning-theoretic tools, including VC dimension and Rademacher complexity.

\subsection{Preliminaries and Notation}
\label{sec: preliminaries and notation}
Let $(\Omega, \sigF, \Prob)$ be the underlying probability space. We consider a standard Borel space $(\samplespace, \sigZ)$ as the sample space. We call the measurable space $(\hypoclass, \sigHypo)$ the hypothesis class, with each $h\in\hypoclass$ a hypothesis. We denote the measurable and bounded measurable functions $\hypoclass \to \R$ by $\measfunc{\hypoclass}$ and $\bounmeasfunc{\hypoclass}$. A loss function is a measurable function $\map{\loss}{\hypoclass \times \samplespace}{[0, \infty)}$. Let $\RVZ$ be a random variable in the sample space $\samplespace$ and let $\RVsample \coloneqq \{\RVZ_i\}_{i=1}^\samplesize$ be $\samplesize$ i.i.d. copies of $\RVZ$. Given a loss function $\loss$, the (expected) risk $\risk$ is defined as
\begin{align}
\label{definition: risk}
    \map{\risk}{\hypoclass}{[0, \infty)}, \quad \risk(h) &\coloneqq \Expval_\RVZ [\loss(h, \RVZ)]\,.
\end{align}
The empirical risk $\emprisk$ of a hypothesis $h$ and a set of points $\{z_i\}_{i=1}^\samplesize \in \samplespace^\samplesize$ is given by
\begin{align}
    \label{definition: empirical risk}
    &\map{\emprisk}{\hypoclass \times \samplespace^\samplesize }{[0, \infty)}\,, \nonumber \\
    &\emprisk \left(h, \{z_i\}_{i=1}^\samplesize\right) \coloneqq \frac{1}\samplesize\sum_{i=1}^\samplesize \loss(h,  z_i)\,.
\end{align}
We call the difference of the risk and the empirical risk the generalization gap, which we denote by 
\begin{align*}
    &\map{\Delta}{\hypoclass \times \samplespace^\samplesize }{\R},
    \\
    &\gengap (h) \coloneqq \Delta(h, \{z_i\}_{i=1}^\samplesize) \coloneqq \risk(h)-\emprisk(h, \{z_i\}_{i=1}^\samplesize) \,.
\end{align*}
For a measure $\mu$ on a measurable space $(\spaceS,\sigS)$ and an integrable function $\map{f}{\spaceS}{\R}$, we use the shorthand
\begin{align} \label{eq: operator-notation}
    \mu[f] \coloneqq \int_\spaceS f \, d\mu\,.
\end{align}
When $\mu$ is the probability measure induced by a random variable $\RVZ$, denoted by $\Prob_\RVZ$, we instead write $\Expval_\RVZ[f]$. We next define OT discrepancies between probability measures $\mu$ and $\nu$ on Polish spaces $\spaceS$ and $\spaceT$, respectively. Let the cost function $\costfunc\in L^1(\mu \otimes \nu)$ be bounded from below. The OT distance is given by
\begin{align*}
    \opttransdist{\mu}{\nu} \coloneqq \inf_{\pi \in \Pi (\mu, \nu)} \pi[\costfunc] \,,
\end{align*}
where $\Pi (\mu, \nu)$ is the set of couplings of $\mu$ and $\nu$, i.e., the set of joint probability measures with marginals $\mu$ and $\nu$. The Kantorovich duality provides an equivalent dual formulation and will be central to our main result.

\begin{theorem}\citep[Theorem 1.3]{villani2021topics}
\label{theorem: kantorovich duality}
    Let the cost function $\costfunc\in L^1(\mu \otimes \nu)$ be bounded from below. Then we have
    \begin{align*}
        \opttransdist{\mu}{\nu} = \sup_{\substack{\psi_1 \in \bounmeasfunc{\spaceS}, \psi_2 \in \bounmeasfunc{\spaceT}:\\ \psi_1 \oplus \psi_2 \leq \costfunc}} \mu [\psi_1] + \nu[\psi_2] \, ,
    \end{align*}
    where $\psi_1 \oplus \psi_2 (x,y) \coloneqq \psi_1 (x) + \psi_2(y)$.
\end{theorem}

The PAC-Bayesian bound due to McAllester will serve as a baseline example to demonstrate how to construct PAC-Bayesian bounds that use OT distances instead of the KL divergence. Since the KL divergence is a central component of the classical bound, we first recall its definition.  Let $\mu$ and $\nu$ be probability measures on a measurable space $(\spaceS,\sigS)$. The KL divergence of $\mu$ w.r.t.\ $\nu$ is defined as
\begin{align*}
    \KL{\mu}{\nu} \coloneqq
    \begin{cases}
        \displaystyle
        \mu\left[\ln\left(\frac{d\mu}{d\nu}\right)\right],
        & \text{if } \mu \ll \nu,\\ 
        \infty,
        & \text{otherwise},
    \end{cases}
\end{align*}
where $\frac{d\mu}{d\nu}$ denotes the Radon–Nikodym derivative of $\mu$ w.r.t.\ $\nu$. In particular, the KL divergence is finite only if $\mu$ is absolutely continuous w.r.t.\ $\nu$.

\begin{theorem}\citep{mcallester2003pacbayes}
\label{theorem: McAllester}
    Let $\map{\loss}{\hypoclass \times \samplespace}{[0, 1]}$ be a measurable loss function, let $\prior$ be a prior distribution over $\hypoclass$, $\samplesize \geq 2$ and let $\delta \in (0,1)$. Then, with probability at least $1-\delta$ over the draw of $\RVsample$,
    \begin{align*}
        \posterior[\risk] \leq \;& \posterior[\emprisk(\cdot, \RVsample)] + \sqrt{\frac{\KL{\posterior}{\prior} + \ln\frac{\samplesize}{\delta}}{2(\samplesize-1)}}
    \end{align*}
    holds simultaneously for any distribution $\posterior$ over $\hypoclass$.
\end{theorem}

\section{Main Result: Optimal Transport based PAC-Bayesian Bound}
\label{section: optimal transport based tound}

In this section, we present a general template for PAC-Bayesian generalization bounds in which the KL divergence is replaced by OT distances. Before stating the general bound, we first introduce and discuss a crucial regularity assumption that characterizes the relation between the learning task and the cost function. This condition needs to be verified for the specific learning task and transport cost under consideration.

\begin{assumption}[regularity condition]
        \label{assumption: regularity constant}
        Let the cost function $\map{\costfunc}{\hypoclass \times \samplespace}{[0, \infty)}$ be a measurable function, and $n\in\N$. Assume that there exists $\constgengap > 0$ such that
        \begin{align}
            \label{eq: assumption regularity constant}
            \gengap^2 \oplus - \gengap^2 \leq \constgengap \cdot \costfunc
        \end{align}
        for $\Prob_\RVsample$ almost all $\sample \equiv \{z_i\}_{i=1}^\samplesize \in \samplespace^\samplesize$.
\end{assumption}

Assumption~\ref{assumption: regularity constant} links the learning problem and the OT cost. It ensures that the geometry induced by the cost function $\costfunc$ is meaningful for the underlying learning task. Depending on the samplesize $\samplesize$, the difference of the generalization gap for hypotheses $h$ and $f$ is controlled by their ground cost $\costfunc(h,f)$. This condition is not automatic and must be verified on a case-by-case basis for each specific choice of transport cost. Similar assumptions were made for example by \cite{amit2022integral} and \cite{raginsky2026separating}. While \cite{amit2022integral} use such regularity assumptions to derive PAC-Bayesian bounds based on integral probability metrics, \cite{raginsky2026separating} use related sensitivity assumptions to derive deterministic generalization bounds from the geometry of the learning problem. In particular, they interpret generalization as a sensitivity problem, separating the geometric properties of the learning problem from the probabilistic assumptions on the data. For bounded losses as in McAllester-type bounds, the assumption can be weakened. If $\map{\loss}{\hypoclass \times \samplespace}{[0,1]}$, then $|\gengap(f)| \leq 1$. Hence, it suffices to establish $|\gengap(f)-\gengap(g)| \leq \constgengap \costfunc(f,g)$, since it implies the regularity condition \eqref{eq: assumption regularity constant}. Based on the regularity condition, we derive our following main theorem.

\begin{theorem}
    \label{theorem: optimal transport PAC-Bayes bound}
    Let $\map{\loss}{\hypoclass \times \samplespace}{[0, 1]}$ be a measurable loss function. Let $\prior$ be a prior distribution over $\hypoclass$, let $\samplesize \geq 2$ and let $\delta \in (0,1)$. Suppose that Assumption~\ref{assumption: regularity constant} holds. Then, with probability at least $1-\delta$ over $\RVsample$,
    \begin{align}
        \label{eq: OT PAC-bayes bound}
        \posterior[\risk] \leq \posterior[\emprisk(\cdot, \RVsample)] + \sqrt{\constgengap \opttransdist{\posterior}{\prior} + \frac{\ln\frac{\samplesize}{\delta}}{2(\samplesize-1)}}
    \end{align}
    holds simultaneously for any distribution $\posterior$ over $\hypoclass$, such that $c \in L^1(\prior \otimes \posterior)$.
    In particular, if $\constgengap = \frac{1}{2(\samplesize - 1)}$, then
    \begin{align}
        \label{eq: OT PAC-bayes bound reg const decr}
        \posterior[\risk] \leq \posterior[\emprisk(\cdot, \RVsample)] + \sqrt{\frac{\opttransdist{\posterior}{\prior} + \ln\frac{\samplesize}{\delta}}{2(\samplesize-1)}}
    \end{align}
    holds with probability at least $1-\delta$ over $\RVsample$ simultaneously for any distribution $\posterior$ over $\hypoclass$, such that $c \in L^1(\prior \otimes \posterior)$.
\end{theorem}

We prove the theorem in the Supplementary Material~\ref{appendix: proof of theorem optimal transport PAC-Bayes bound}. The regularity constant $\constgengap$ plays a crucial role in the growth of the complexity term. If $\constgengap$ increases with $\samplesize$, the bound in \eqref{eq: OT PAC-bayes bound} is only valuable for finite $\samplesize$. However, if $\constgengap$ decreases sufficiently fast with $\samplesize$, \eqref{eq: OT PAC-bayes bound} gives also an asymptotic convergence guarantee as in \eqref{eq: OT PAC-bayes bound reg const decr}. This analysis the subject of Section~\ref{sec: cases of vanishing regularity constant}, where we discuss explicit instances of the template based on OT metrics, including, OT distances with concave costs, the TV distance and Wasserstein-$p$ distances.

Theorem~\ref{theorem: optimal transport PAC-Bayes bound} closely mirrors the classical PAC-Bayesian bound of McAllester in Theorem~\ref{theorem: McAllester}. In particular, the structure of the bound is similar, with the KL divergence being replaced by an OT distance, and the regularity constant $C_n$ which now appears to capture the geometry of the hypothesis space. Consequently, other PAC-Bayesian bounds that are traditionally formulated in terms of the KL divergence can be transformed by substituting the KL divergence with the general OT distance. This includes, for example, the bounds of \cite{catoni2007pacbayes}, \cite{maurer2004note} and \cite{seeger2003pacbayes}. At proof level, the key argument is modifying the variational argument underlying the KL-based bound. Specifically, the Donsker-Varadhan variational formula is replaced by the Kantorovich duality for the OT distance in Theorem~\ref{theorem: kantorovich duality}. See the Supplementary Material~\ref{appendix: proof of theorem optimal transport PAC-Bayes bound} for details.

The use of OT distances in this context offers several advantages. Most notably, OT distances do not require the posterior distribution to be absolutely continuous w.r.t.\ the prior, thereby relaxing a key restriction imposed by KL-based PAC-Bayes bounds and allowing for greater flexibility in the choice of priors. Therefore, OT distances can compare mutually singular measures, including deterministic priors and posteriors represented by Dirac measures. Moreover, OT distances explicitly incorporate the underlying geometry of the hypothesis space, which can lead to tighter and more informative bounds.
OT distances and IPMs both provide alternatives to KL divergence that remain finite for mutually singular distributions, but they encode different notions of discrepancy. IPMs measure differences through expectations over a prescribed class of test functions and thus depend critically on its choice. In contrast, OT distances compare distributions through couplings and an explicit cost on the underlying hypothesis space. Consequently, OT distances incorporate the geometry of the hypothesis space directly, whereas IPMs require this geometry to be encoded implicitly through the choice of the test function class. This distinction is particularly relevant for PAC-Bayesian bounds as an OT-based bound can exploit a meaningful cost between hypotheses without requiring a corresponding test function class, providing a direct way to control how far the posterior moves from the prior in hypothesis space.

\section{Cases of Vanishing Regularity Constant}
\label{sec: cases of vanishing regularity constant}
In our bound in Theorem~\ref{theorem: optimal transport PAC-Bayes bound}, the complexity term vanishes asymptotically only if the regularity constant $\constgengap \to 0$ as $n \to \infty$. In this section, we analyze several such instances using classical learning theory.

\subsection{Concave OT Bound for Linear Regression}
First, we consider linear regression and derive PAC-Bayesian bounds for concave OT distances, a special case of the $\phi$-OT distances defined below.

\begin{definition}[$\phi$-Optimal Transport]
\label{definition: concave optimal transport}
    Let $(\hypoclass,d)$ be a Polish space equipped with its Borel $\sigma$-algebra $\borel{\hypoclass}$ and let $\mu$ and $\nu$ be probability measures on $\hypoclass$. Let $\map{\phi}{[0,\infty)}{[0,\infty)}$ be non-decreasing with $\phi(0)=0$ and $\phi(r)>0$ for $r>0$. The \textit{$\phi$-OT distance} between $\mu$ and $\nu$ is defined as
    \begin{align*}
        \phiOT{\mu}{\nu} \coloneqq \inf_{\pi \in \Pi(\mu,\nu)} \pi [\phi \circ d] \, .
    \end{align*}
\end{definition}

If $\phi$ is additionally concave, we refer to $\mathsf{C}_\phi$ as a concave OT distance. These distances use the natural geometry of $\hypoclass$, while the choice of $\phi$ controls the relative cost of local and long-range transport. Hence for concave $\phi$, long displacements incur a smaller cost per unit distance than short ones. Moreover, concavity implies subadditivity, and hence $d_\phi\coloneqq\phi\circ d$ is itself a metric. Examples include
\begin{align*}
    &\phi_{\alpha}(r) = r^\alpha, \quad 0<\alpha\leq 1 \, , \qquad
    \phi_{\ln}(r) = \ln\left(1+r\right) \\
    &\text{and} \quad \phi_{\exp}(r) = 1-\exp\left(-r\right) \, .
\end{align*}
For $\phi(r)=r$, our regularity condition (Assumption~\ref{assumption: regularity constant}) is $\constgengap$-Lipschitz continuity of $\gengap^2$ w.r.t.\ $d$ and the corresponding OT distance is the Wasserstein-$1$ distance. We thus recover precisely the regularity assumption of \cite{amit2022integral}.  More generally, for concave $\phi$, the assumption is equivalent to $\constgengap$-Lipschitz continuity of $\gengap^2$ w.r.t.\ the transformed metric $d_\phi \coloneqq \phi \circ d$.

The following theorem instantiates Theorem~\ref{theorem: optimal transport PAC-Bayes bound} for linear regression and in particular, bounds the regularity constant at rate $O(\samplesize^{-1})$. Although PAC-Bayesian bounds for linear regression have been considered previously, we are the first, to the best of our knowledge, to derive such a bound using general OT distances in place of the KL divergence.

\begin{theorem}[Concave OT PAC-Bayes for Linear Regression]
\label{theorem: linear regression concave OT}
    Let $\map{\phi}{[0,\infty)}{[0,\infty)}$ be non-decreasing and concave, with $\phi(0)=0$ and $\phi(r)>0$ for $r>0$. Let $\RVZ = (\RVX, \RVY)$ be a random variable in $\R^d\times\R$ satisfying
    \begin{align*}
        \|\RVX\|_2\leq R \qquad\text{and}\qquad |\RVY|\leq Y_{\max}
    \end{align*}
    almost surely for some $R,Y_{\max}>0$.
    Consider the hypothesis class $\hypoclass = \left\{ w\in\R^d:\| w\|_2\leq B\right\}$ for $B>0$, equipped with the Euclidean metric $d(w,v)=\| w-v\|_2$ and squared loss
    \begin{align*}
        \ell(w,(x,y)) = \frac{1}{2}\left(\langle w,x\rangle-y\right)^2.
    \end{align*}
    Let $\delta\in(0,1)$, $\samplesize \geq 2$ and $\prior$ be a prior distribution over $\hypoclass$. Then, with probability at least $1-\delta$ over $\RVsample$,
    \begin{multline*}
        \posterior[\risk] \leq \posterior[\emprisk(\cdot,\RVsample)] 
        + \Bigg( c(B,R,Y_{\max})\\
        \frac{\left(1 + \sqrt{\ln \frac{12}{\delta}}\right)^2}{n} \phiOT{\prior}{\posterior}
        + \frac{\ln\frac{2\samplesize}{\delta}}{2(\samplesize-1)} \Bigg)^\frac{1}{2}
    \end{multline*}
    holds simultaneously for every posterior distribution $\posterior$ over $\hypoclass$, with $c(B,R,Y_{\max}) = \frac{4B}{\phi(2B)} \cdot \left( \frac{B^2R^2 + Y_{\max}^2}{2} + BR\RVY_{\max}\right) \left( BR^2+RY_{\max} \right)$.
\end{theorem}

For the proof, see the Supplementary Material~\ref{appendix: proof of theorem linear regression concave OT}.

\subsection{Total-Variation Bound}
As a second instance, we specialize the OT framework to the TV distance
\begin{align*}
\TVdist{\mu}{\nu} \coloneqq \sup_{A \in \borel{\hypoclass}} |\mu(A)-\nu(A)| , .
\end{align*}
Since the TV distance is an IPM, PAC-Bayesian bounds based on TV have been considered previously. Our contribution is to systematically combine such bound with classical complexity measures to control the regularity constant in Theorem~\ref{theorem: optimal transport PAC-Bayes bound}. In particular, Theorem~\ref{theorem: rademacher complexity TV PAC-Bayes} controls the regularity constant in terms of the Rademacher complexity. While alternative derivations may also yield bounds of this form, to the best of our knowledge, our framework is the first to explicitly establish this connection between TV-based PAC-Bayesian bounds and Rademacher complexity. Moreover, Theorem~\ref{theorem: vc dimension PAC-Bayes} shows that the regularity constant can alternatively be controlled by the VC dimension, recovering existing TV-based PAC-Bayesian bounds as a special case of our framework. First, we define the Rademacher complexity.

\begin{definition}[Expected Rademacher complexity]
    \label{definition: expected Rademacher complexity}
    Let $(\hypoclass,d)$ be a Polish space equipped with its Borel $\sigma$-algebra $\borel{\hypoclass}$. Let
    \begin{align*}
        \mathcal{F}_{\loss} = \left\{ \loss_h \,\middle|\, \map{\loss_h}{\samplespace}{[0,1]}, \quad \loss_h(z)=\loss(h,z), \quad h\in\hypoclass \right\}
    \end{align*}
    be the loss class induced by $\hypoclass$ and $\loss$. Let $\sigma_1,\ldots,\sigma_\samplesize \sim \text{UNIF}\{-1,1\}$ be i.i.d.\ and independent of $\RVsample$. The \textit{expected Rademacher complexity} of $\mathcal{F}_{\loss}$ is defined as
    \begin{align*}
        \rademachercompl_{\samplesize}(\mathcal{F}_{\loss}) \coloneqq \Expval_{\RVsample,\sigma} \left[ \sup_{h\in\hypoclass} \frac{1}{\samplesize} \sum_{i=1}^{\samplesize} \sigma_i\loss(h,\RVZ_i) \right].
    \end{align*}
\end{definition}

\begin{theorem}[Total-variation PAC-Bayes bound based on Rademacher complexity]
\label{theorem: rademacher complexity TV PAC-Bayes}
    Let $\map{\loss}{\hypoclass\times\samplespace}{[0,1]}$ be a measurable loss function and let $\mathcal{F}_{\loss}$ be the corresponding induced loss class.
    Further, let $\prior$ be a prior distribution over $\hypoclass$, $\samplesize \geq 2$ and let $\delta\in(0,1)$. Then, with probability at least $1-\delta$ over $\RVsample$,
    \begin{multline*}
        \posterior[\risk]
        \leq \posterior[\emprisk(\cdot,\RVsample)] \\
        + \sqrt{ \left( 2\rademachercompl_{\samplesize} (\mathcal{F}_{\loss})
        + \sqrt{ \frac{\ln(4/\delta)} {2\samplesize} } \right)^2 \TVdist{\posterior}{\prior} + \frac{\ln(2\samplesize/\delta)}{2(\samplesize-1)}}
    \end{multline*}
    holds simultaneously for every posterior distribution $\posterior$ over $\hypoclass$.
\end{theorem}

The proof is given in the Supplementary Material~\ref{appendix: proof of theorem rademacher complexity TV PAC-Bayes}.

\begin{remark}
    The Rademacher complexity $\rademachercompl_{\samplesize}(\mathcal{F}_{\loss})$ can be bounded for many standard loss classes. Here, $\rademachercompl_{\samplesize}(\mathcal{F}_{\loss})$ denotes the complexity of the loss class, not the hypothesis class $\hypoclass$. In many standard settings, these two complexities can be related. For example, for Lipschitz loss, the contraction principle provides a direct connection between the complexity of the loss class and that of $\hypoclass$ and similarly in binary classification with the $0$-$1$ loss. See, e.g., \citet{shalev2014understanding, mohri2018foundations}. The resulting rate is typically $O(1/\sqrt{\samplesize})$.
\end{remark}

A related concept is the VC dimension, which provides the combinatorial complexity of binary hypothesis classes. For classes with finite VC dimension, the Rademacher complexity is of order $O(\sqrt{\ln\samplesize/\samplesize})$ \citep{shalev2014understanding}. Thus, the VC dimension can be used in the preceding approach to show that the regularity constant $\constgengap$ vanishes. As shown below for the TV example, the it can also yield a direct bound on $\constgengap$. We first recall its definition.

\begin{definition}
    \label{definition: VC dimension}
    Let $\hypoclass$ be a class of binary-valued functions $\map{h}{\inspace}{\{0,1\}}$. A finite set $\{x_1,\ldots,x_m\}\sub\inspace$ is \textit{shattered} by $\hypoclass$ if, for every $(y_1,\ldots,y_m)\in\{0,1\}^m$, there exists $h\in\hypoclass$ such that
    $ h(x_i)=y_i$ for all $i=1,\ldots,m$.
    The \textit{VC dimension} of $\hypoclass$, denoted by $\VCdim(\hypoclass)$, is the largest $m$ for which there exists a set of $m$ points that is shattered by $\hypoclass$. If sets of arbitrarily large finite cardinality can be shattered, we set $\VCdim(\hypoclass)=\infty$.
\end{definition}

In our framework, we can recover the bound by \cite{amit2022integral}.

\begin{theorem}
\label{theorem: vc dimension PAC-Bayes}
    Consider a binary classification problem, with the 0-1 loss, and hypothesis class $\hypoclass$, with finite VC dimension. Then there exists some constant $c > 0$ s.t. for any prior $\prior$ and $\delta \in(0,1)$, with probability at least $1 - \delta$ over the samples $\RVsample$
    \begin{multline*}
        \posterior[\risk]
        \leq \posterior[\emprisk(\cdot,\RVsample)] \\
        + \sqrt{ c \frac{\VCdim(\hypoclass) + \ln \left(\frac{1}{\delta}\right)}{\samplesize} \TVdist{\posterior}{\prior} + \frac{\ln(2\samplesize/\delta)}{2(\samplesize-1)}}
    \end{multline*}
\end{theorem}

The proof is given in the Supplementary Material~\ref{appendix: proof of theorem VC dimension TV PAC-Bayes}.

\begin{remark}
    Using classical learning-theoretic approaches to bound the regularization constant is not limited to the special case of the TV distance. For other OT distances, however, additional assumptions on the hypothesis class and the loss function are required, see Supplementary Material~\ref{appendix: bounding the regularity constant using classical learning theory for general OT distances}.
\end{remark}

\subsection{Wasserstein-\texorpdfstring{$p$}{p} Distance Bounds}
\label{section: p-wasserstein bound}
For $p\geq1$, the Wasserstein-$p$ distance is defined by
\begin{align*}
    \big(\wassersteinpdist{\mu}{\nu}\big)^p
    \coloneqq \inf_{\pi\in\Pi(\mu,\nu)}\pi[d^p].
\end{align*}
Previous PAC-Bayesian bounds based on Wasserstein distances considered the special case $p=1$, which is an IPM. Our framework extends these results to arbitrary $p\geq1$. For compact $\hypoclass$, Theorem~\ref{theorem: optimal transport PAC-Bayes bound}, yields a PAC-Bayesian bound for $\wassersteinpdist{\mu}{\nu}$ for every $p\geq1$. To state the result, define the function class of normalized loss differences
\begin{align*}
    \losshypoclass \coloneqq \left\{\frac{\loss(h, \cdot)-\loss(h', \cdot)}{d(h,h')^p} :h,h'\in\hypoclass,\ h\neq h' \right\}.
\end{align*}
and let $G>0$ be a uniform $p$-Hölder constant of $\loss(\cdot, z)$. We further define the associated entropy integral
\begin{align*}
    J(1,\losshypoclass) \coloneqq \sup_\mu\int_0^1 \sqrt{1+\log \epscovnumb\big(\varepsilon G,\losshypoclass,L_2(\mu)\big)} \,d\varepsilon,
\end{align*}
where the supremum is over finitely supported probability measures $\mu$ on $\samplespace$ and $\epscovnumb(\varepsilon,\losshypoclass, L_2(\mu))$ denotes the $L_2(\mu)$ covering number of $\losshypoclass$ at radius $\varepsilon$.

\begin{theorem}
\label{theorem: finite p-Wasserstein distance PAC-Bayes bound}
    Let $\map{\loss}{\hypoclass \times \samplespace}{[0, 1]}$ be measurable and $(\hypoclass,d)$ be a compact metric space. Assume that $\loss(\cdot, z)$ is a $p$-Hölder continuous with constant $G$ uniformly in $z\in\samplespace$. Let $\prior$ be a prior distribution over $\hypoclass$, $\samplesize \geq 2$, $\delta \in (0,1)$. Then there is a constant $c > 0$ s.t. with probability at least $1-\delta$ over $\RVsample$
    \begin{multline}
        \posterior[\risk] \leq \posterior[\emprisk(\cdot, \RVsample)]
        + \Bigg( \Bigg( \frac{c\,G\,J(1,\losshypoclass)}{\sqrt n} \\
        + G\sqrt{\frac{2\ln(1 / \delta)}{n}} \Bigg) \wassersteinpdist{\posterior}{\prior}^p
        + \frac{\ln(2\samplesize/\delta)}{2(\samplesize - 1)}\Bigg)^\frac{1}{2}
    \end{multline}
    holds simultaneously for any distribution $\posterior$ over $\hypoclass$.
\end{theorem}

The proof is given in the Supplementary Material~\ref{appendix: proof of theorem finite p-Wasserstein distance PAC-Bayes bound}.

\section{Experiments}
\label{sec: experiments}
\begin{figure*}[t]
    \centering
    \begin{subfigure}[b]{0.48\textwidth}
        \centering
        \includegraphics[width=\textwidth]{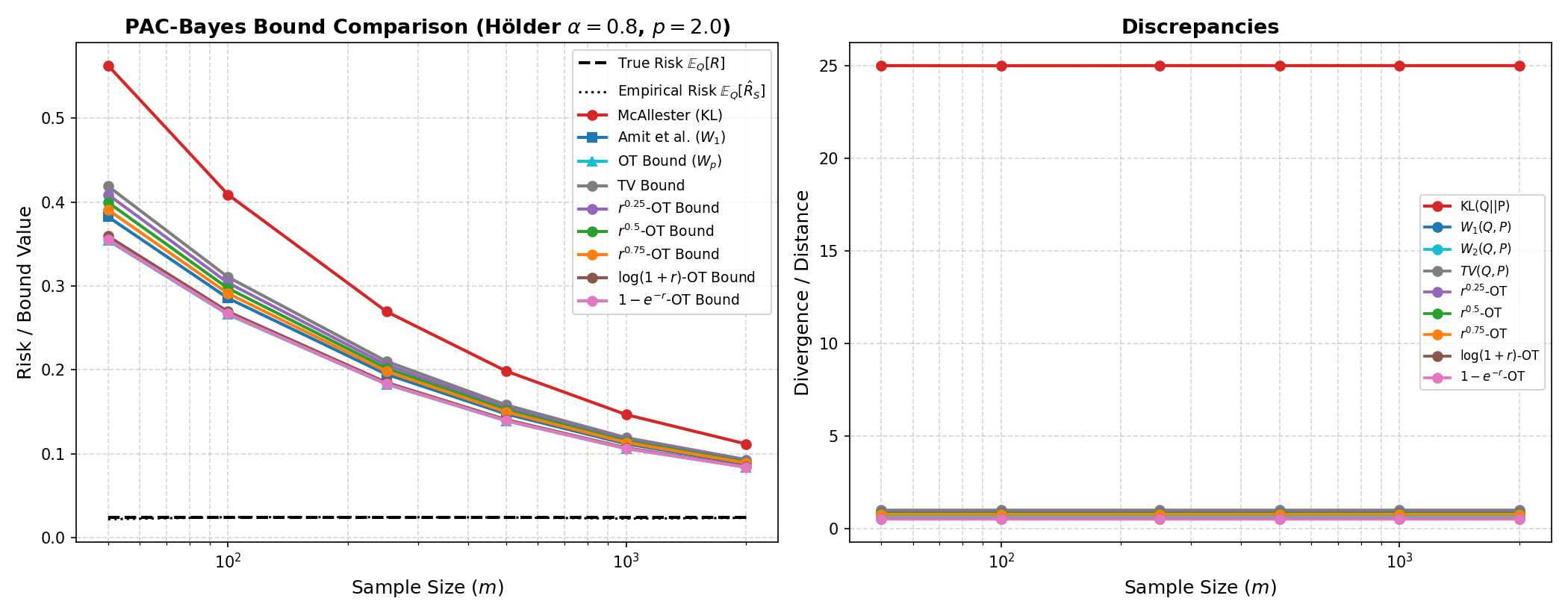} 
        \caption{Sample size $\samplesize$ sweep}
        \label{fig: toy_regr samplesize sweep}
    \end{subfigure}
    \hfill
    \begin{subfigure}[b]{0.48\textwidth}
        \centering
        \includegraphics[width=\textwidth]{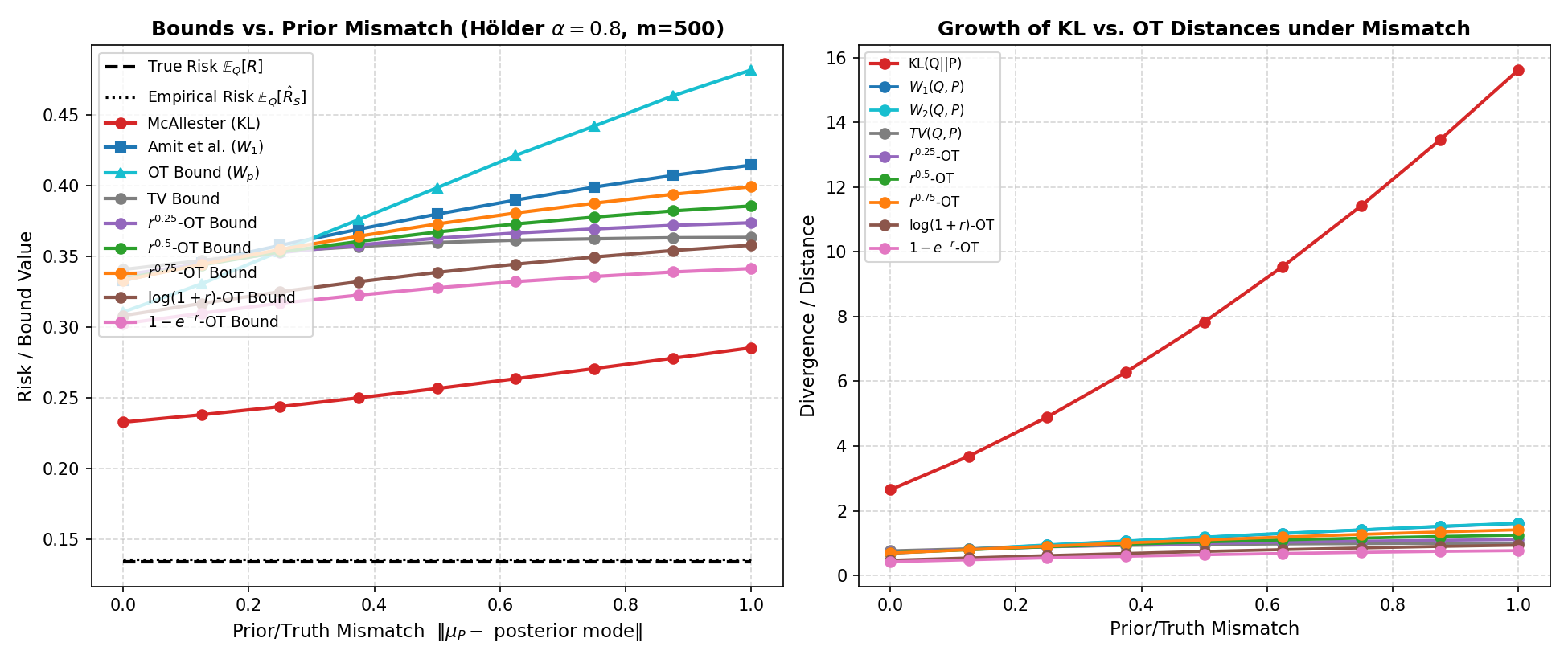}
        \caption{Prior mean sweep}
        \label{fig: toy_regr prior sweep}
    \end{subfigure}
    
    \vspace{0.5cm}
    
    \begin{subfigure}[b]{0.48\textwidth}
        \centering
        \includegraphics[width=\textwidth]{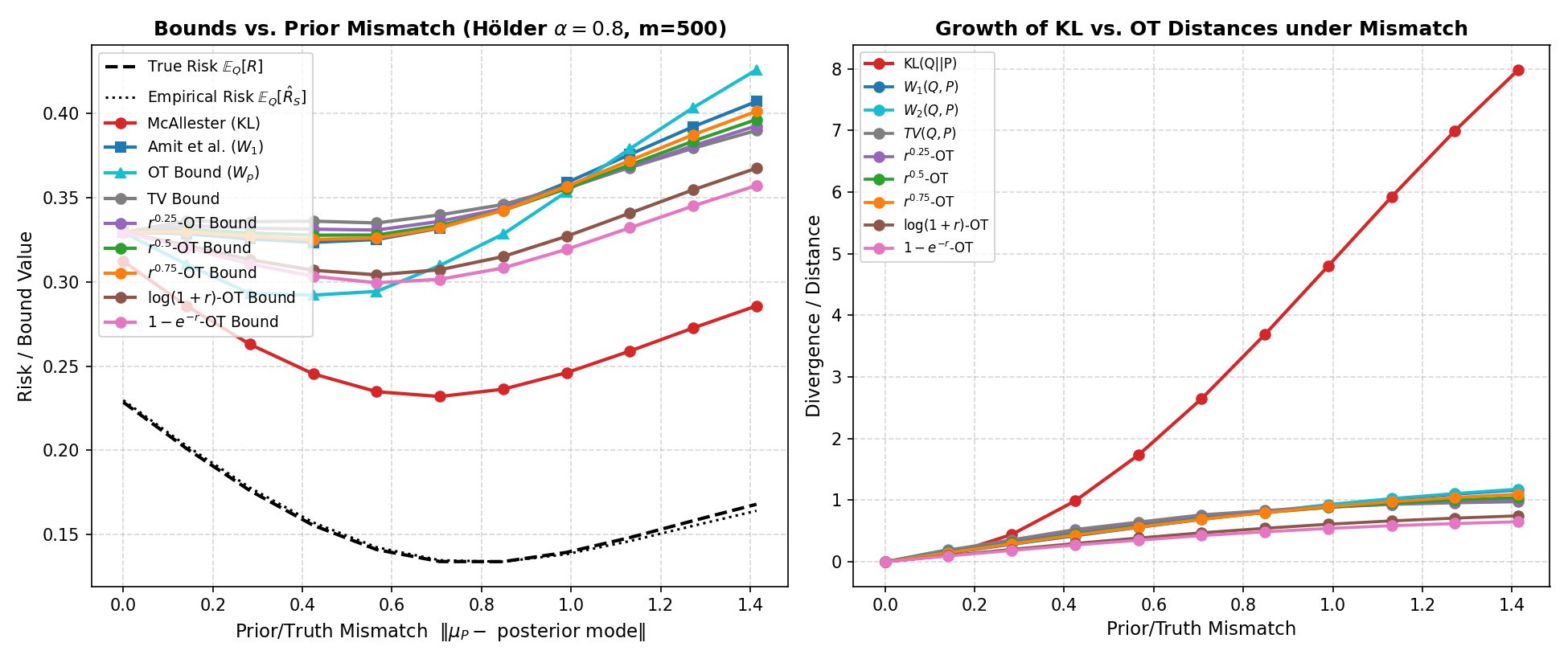}
        \caption{Posterior mean sweep}
        \label{fig: toy_regr posterior sweep}
    \end{subfigure}
    \hfill
    \begin{subfigure}[b]{0.48\textwidth}
        \centering
        \includegraphics[width=\textwidth]{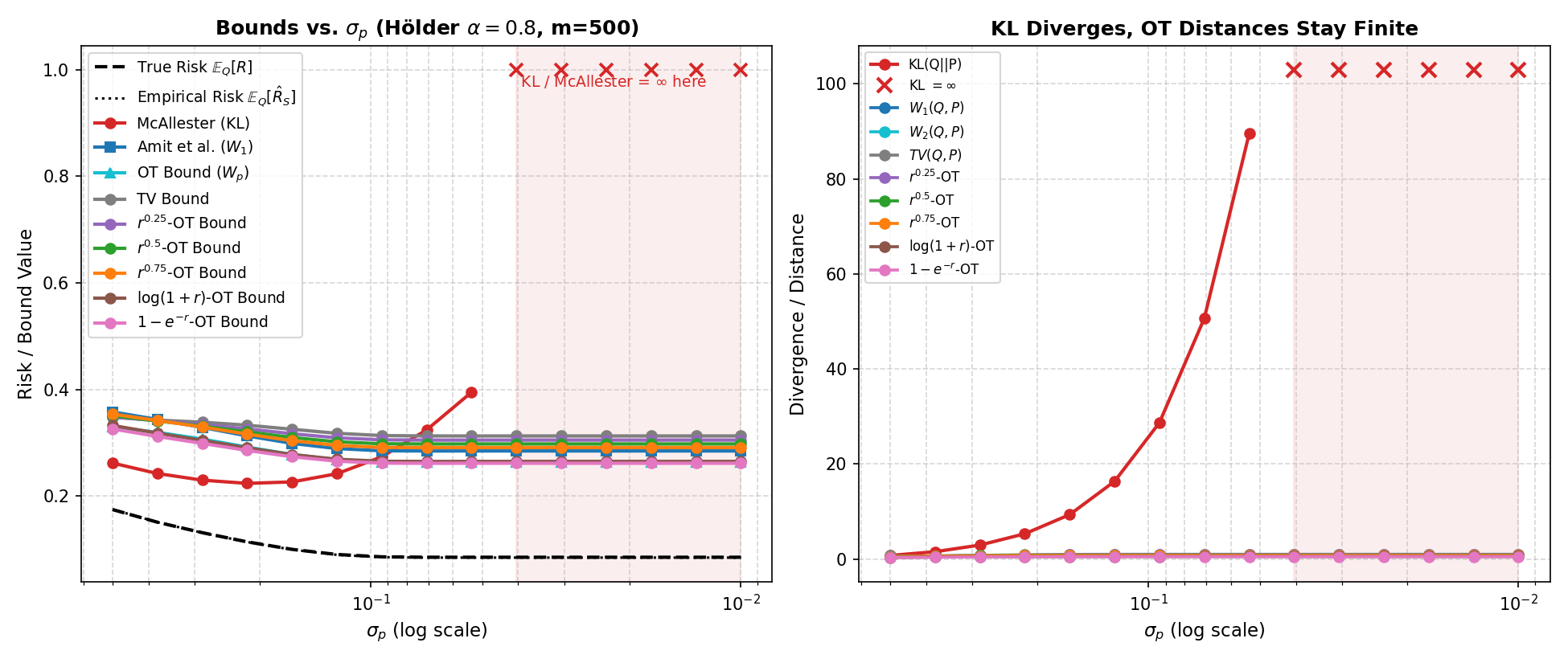}
        \caption{$\sigma$ sweep}
        \label{fig: toy_regr sigma sweep}
    \end{subfigure}

    \caption{In this linear regression experiment, we compare PAC-Bayesian bounds whose remainder terms are based on different statistical distances. We compare the behavior of the bounds for different priors, posteriors and sample sizes. For each experiment, the picture on the left shows the behavior of the distance and the right shows the full bound.}
    \label{fig: toy_regr}
\end{figure*}

We compare McAllester-style PAC-Bayesian bounds based on KL divergence, IPMs, and OT distances. PAC-Bayesian bounds come in several forms, such as the McAllester and Catoni bounds, each with distinct properties and advantages. Rather than comparing these different PAC-Bayesian bound constructions, we fix the McAllester-style framework and vary only the discrepancy measure appearing in the bound. This allows us to isolate how the choice of discrepancy measure affects the resulting generalization bound and to illustrate the distinct properties induced by KL divergence, IPMs, and OT distances. We therefore first study these measures independently of the PAC-Bayesian setting in Supplementary Material~\ref{appendix: properties of the statistical distances}: The KL divergence requires absolute continuity and becomes infinite under support mismatch, whereas IPMs and OT distances remain finite. Further unlike the KL divergence, IPMs and OT distances do not explode under changes in concentration or location of probability mass. In contrast to IPMs, OT distances provide a direct and interpretable notion of how far probability mass has to be moved.

\subsection{Discrete Linear Regression}

The first experiment uses a synthetic dataset. We consider two-dimensional linear regression with inputs sampled uniformly from the Euclidean unit ball and outputs generated according to
\begin{align*}
    y = \langle w_\star,x\rangle + \varepsilon,
    \quad
    w_\star=(0.5,-0.5),
    \quad
    \varepsilon\sim\calN(0,0.1^2).
\end{align*}
The hypothesis class is a finite grid $\hypoclass \sub [-1,1]^2$ of linear predictors $h_w(x)=\langle w,x\rangle$. Priors and posteriors are discretized Gaussian distributions on this grid. We use the bounded Hölder loss
\begin{align*}
    \loss(w,(x,y))=\min\{1,|\langle w,x\rangle-y|/M\}^{\alpha}.
\end{align*}
For each configuration, we compute the empirical Gibbs risk, estimate the population Gibbs risk on an independent test set and compare PAC-Bayesian bounds based on KL divergence, Wasserstein-$1$ and Wasserstein-$2$, TV distance and other OT distances.

We perform four controlled sweeps. First, varying $\samplesize$ in Figure~\ref{fig: toy_regr samplesize sweep} shows similar scaling across all bounds, consistent with their $\mathcal{O}(\sqrt{\ln(\samplesize)/\samplesize})$ complexity terms. In Figure~\ref{fig: toy_regr prior sweep}, we move the prior mean away from a fixed posterior concentrated near $w_\star$, isolating the effect of prior–posterior mismatch. Next, Figures~\ref{fig: toy_regr prior sweep} and~\ref{fig: toy_regr posterior sweep} vary the prior and posterior means, respectively, to isolate the effect of prior-posterior mismatch. Although the KL divergence increases sharply with this mismatch, its bound can remain competitive because its complexity term has a smaller prefactor compared to the OT and IPM based bounds. Finally, Figure~\ref{fig: toy_regr sigma sweep} decreases the common variance $\sigma^2$. The KL term grows rapidly as $\sigma$ decreases and is numerically evaluated as infinite for $\sigma\leq0.05$ due to floating-point overflow, although it remains finite mathematically. The OT-based discrepancies remain finite on the bounded hypothesis grid, yielding meaningful bounds even for small $\sigma$. For larger $\sigma$, however, the KL-based bound is tighter. Performance generally improves as $\phi$ becomes more concave. A possible explanation is that the prior and posterior differ by transporting non-negligible mass over relatively large distances. More concave choices of $\phi$ assign lower marginal cost to such long-range transport, reducing the OT discrepancy and potentially tightening the resulting bound.

\subsection{Classification}
\begin{figure*}[t]
    \centering
    \begin{subfigure}[b]{0.48\textwidth}
        \centering
        \includegraphics[width=\textwidth]{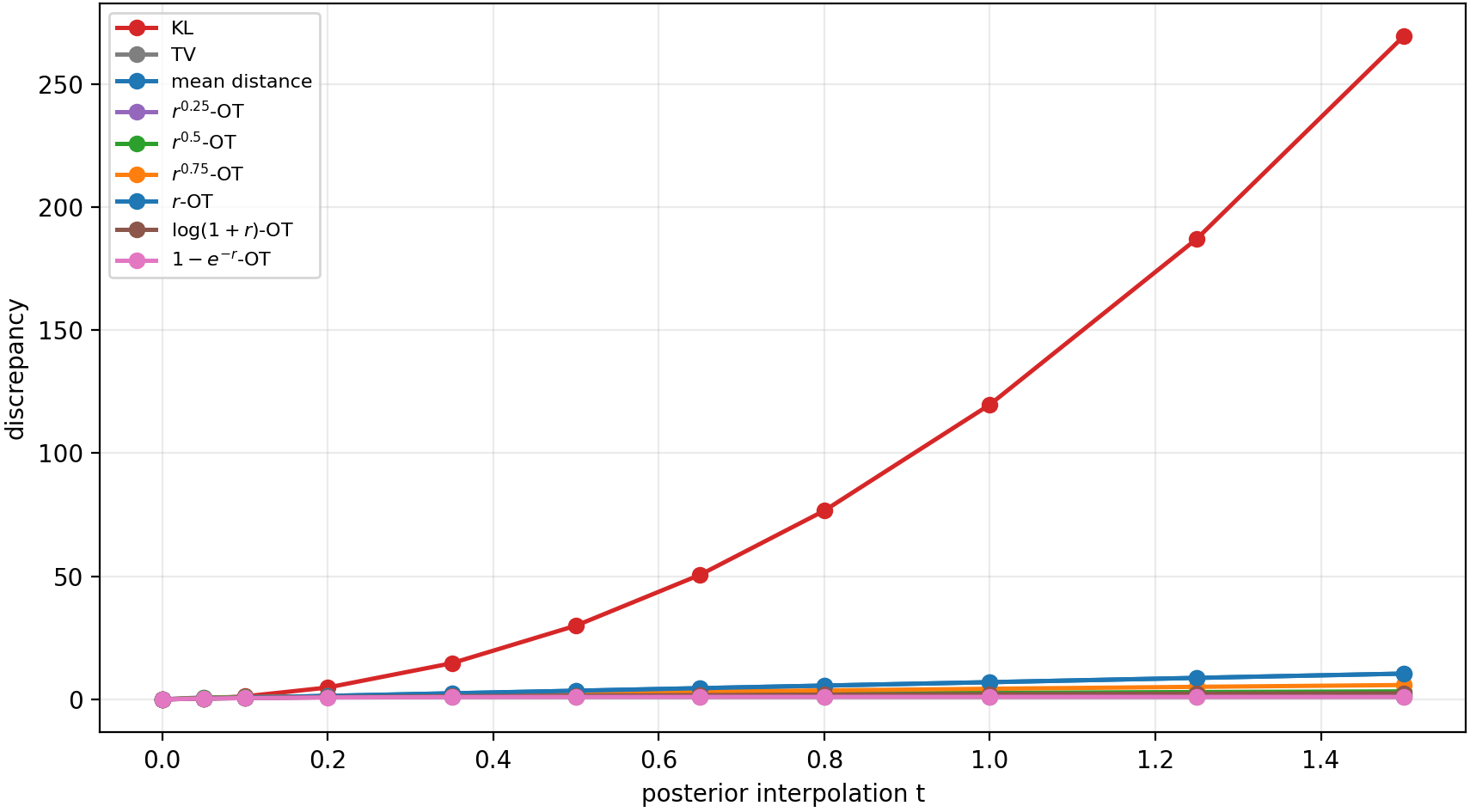}
    \end{subfigure}
    \hfill
    \begin{subfigure}[b]{0.48\textwidth}
        \centering
        \includegraphics[width=\textwidth]{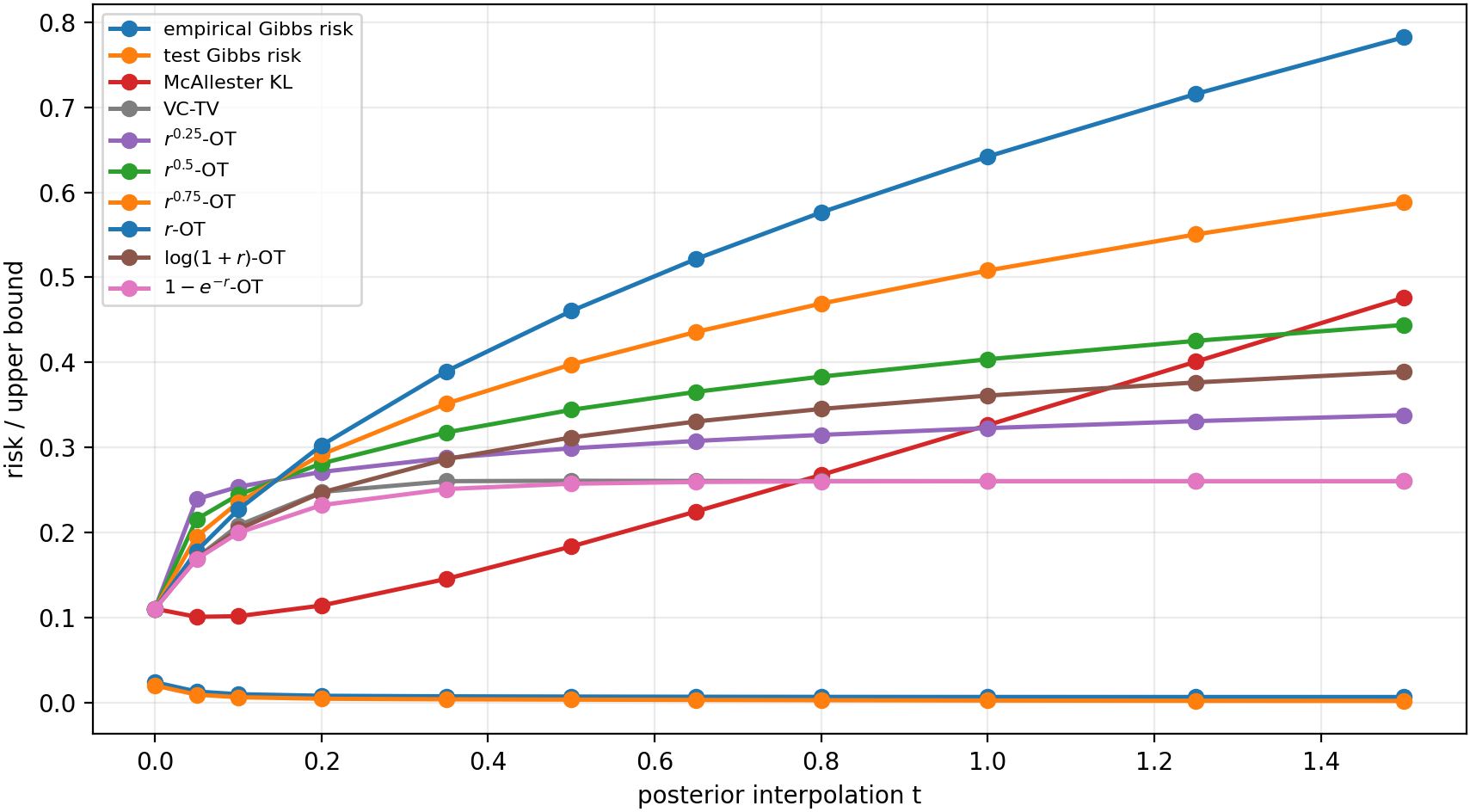} 
    \end{subfigure}
    
    \caption{In this experiment, we compare PAC-Bayesian bounds whose complexity terms are based on different statistical distance. We compare the behavior of the bounds for different posteriors. The picture on the left shows the behavior of the distance and the right shows the full bound.}
    \label{fig: continuous head}
\end{figure*}

We next consider binary classification on the UCI Pendigits dataset \citep[licensed under CC BY 4.0.]{alpaydin1996penbased}. We restrict the task to digits $3$ and $5$, encoded as labels $-1$ and $+1$. We split the data into prior, training, and test sets with proportions $40\%/30\%/30\%$. A two-layer ReLU network is trained as a feature extractor $f$ on the prior set and then frozen. The resulting hypothesis class consists of affine classifiers on the learned representation,
\begin{align*}
    h_\theta(x) = \operatorname{sign}\left( \theta^\top
        \begin{pmatrix}
            f(x) \\
            1
        \end{pmatrix}
        \right),
    \qquad
    \theta\in\R^{d+1}.
\end{align*}

Unlike in the synthetic regression experiment, we use continuous prior and posterior distributions. The prior is the isotropic Gaussian
\begin{align*}
    \prior=\calN(\mu_{\prior},\sigma^2\idmat_{d+1}), 
\end{align*}
where $\mu_{\prior}$ is obtained by training on the prior set. We obtain a target posterior mean $\mu_\star$ by training on the training set and consider the family of posteriors
\begin{align*}
    \posterior_t =\calN\bigl(\mu_{\prior} + t(\mu_\star-\mu_{\prior}),\sigma^2\idmat\bigr), \qquad t\in[0,1.5].
\end{align*}
Values $t \leq 1$ interpolate between the prior mean and the target, while $t > 1$ extrapolates beyond the target to stress-test the discrepancy terms.

For each $t$, we compute the empirical Gibbs risk on the training set and estimate the population Gibbs risk on the test set. We compare McAllester's KL bound with the VC-based bounds involving $\TVdist{\posterior_t}{\prior}$ and $\phiOT{\posterior_t}{\prior}$. See Figure~\ref{fig: continuous head} for the results.

For this experiment, the behavior of the bounds as a function of $t$ can be understood from the exact expressions for KL and $W_p$ given in the Supplementary Material~\ref{appendix: properties of the statistical distances}. In particular, since $\sigma_\prior = \sigma_{\posterior_t}$
\begin{align*}
    \KL{Q_t}{P} &= \frac{t^2\|\mu_\star-\mu_P\|_2^2}{2\sigma^2}, \\
    W_p(Q_t,P) &= t\|\mu_\star-\mu_P\|_2.
\end{align*}
Hence, the KL divergence grows quadratically in $t$, whereas Wasserstein distances grow linearly. For small $t$, the KL-based bound is tighter because its discrepancy term has a smaller coefficient. As $t$ increases, the growth of KL eventually dominates, and the OT-based bounds become tighter. In our experiment, this occurs for $t\geq0.8$. Notably, the KL divergence already exceeds the OT discrepancies in the small-$t$ regime, but the smaller coefficient in the KL bound compensates for this difference. Overall, the results illustrate the complementary behavior predicted by the theory, that KL-based bounds are advantageous for small prior-posterior mismatch, whereas OT-based bounds are more robust to larger shifts.

\section{Conclusion}
We introduced a general framework for PAC-Bayesian generalization bounds that replaces the KL divergence between posterior and prior with general OT distances. This change removes the requirement that the posterior has to be absolutely continuous w.r.t.\ the prior and consequently, allows PAC-Bayesian bounds to be formulated for prior and posterior distributions with disjoint supports. More generally, OT distances account for the geometry of the hypothesis space through the choice of cost function, providing a notion of proximity between distributions that is not captured by the KL divergence or IPMs. The resulting bounds can therefore remain meaningful in settings where KL-based bounds become vacuous due to support mismatch or substantial prior-posterior displacement.

Our numerical experiments illustrate this difference. When the prior is displaced from the posterior or assigns little mass to the region containing the posterior, KL-based bounds become substantially worse, whereas OT-based bounds account for the distance between the corresponding hypotheses and can remain informative. Conversely, when the prior and posterior are well aligned, the KL divergence can provide tighter bounds than the OT-based alternatives. This demonstrates that OT does not uniformly dominate KL, rather, the two complexity measures capture different notions of prior-posterior discrepancy. The choice between them therefore depends on the geometry of the hypothesis space and on how the prior is constructed.

\subsection*{AI use statement}

In this work, we used generative AI tools for assist in the writing of proofs, design or provide feedback on research methodology or experiments, implement methods. We have not used generative AI tools for generate synthetic data sets, help develop theoretical models or conceptual frameworks, formulate mathematical claims, provide critical ingredients for proving mathematical claims, propose or refine hypotheses and assist with translation, clean and reformat datasets, support qualitative and thematic data analysis, interpret results
are not applicable to this work.
Additionally, we used generative AI tools for suggest experimental parameters, create or edit software code, summarize or analyze existing literature, brainstorming, sourcing/searching for information, edit a research paper to improve readability, identify relevant literature.
We have reviewed all AI-assisted work.
LLM-generated code was verified and tested for correctness by two authors. LLM-assisted literature summaries were cross-checked against the original source papers. Any citations or claims derived from these summaries were independently verified to ensure accuracy and appropriate attribution. All information retrieved through AI-assisted search was traced back to primary sources. We confirmed the existence, accuracy, and relevance of each source before incorporating it into the manuscript. AI-assisted edits were reviewed line by line by the authors. No substantive claims, technical content, or logical arguments were altered without explicit author approval.
We take responsibility for the final content of this work, including text, claims, or artifacts produced with the aid of generative AI.

\subsection*{Acknowledgements}
We  acknowledge funding by the ANR-DFG joint project TRINOM-DS under the numbers ANR-20-CE92-0037-01 and OC150/5-1.

\bibliography{references}

\clearpage
\appendix
\thispagestyle{empty}

\onecolumn
\aistatstitle{Optimal Transport PAC-Bayes Bounds: \\
Supplementary Materials}
\label{appendix}

\section{Proof of Theorem~\ref{theorem: optimal transport PAC-Bayes bound}}
\label{appendix: proof of theorem optimal transport PAC-Bayes bound}

\begin{proof}
    Let $\samplesize \geq 2$ and define
    \begin{align*}
        \map{f}{\samplespace^\samplesize}{\R}, \qquad f(\sample) \coloneqq \sup_\posterior \left( 2(\samplesize - 1) \posterior \left[\Delta(\cdot, \sample)^2 \right] - \constgengap \cdot 2(\samplesize - 1) \opttransdist{\posterior}{\prior} \right) \,.
    \end{align*}
    Markov's inequality \citep[Lemma 5.1]{kallenberg2002foundations} yields
    \begin{align}
    \label{eq: proof opt. trans., markov ineq.}
        \Prob[f(\RVsample) \geq \varepsilon ] = \Prob\left[e^{f(\RVsample)} \geq e^{\varepsilon} \right] \leq \frac{\Expval_\RVsample\left[ e^{f}\right]}{e^\varepsilon}.
    \end{align}
    The rest of the proof bounds $\Expval_\RVsample\left[ e^{f}\right]$.
    
    By assumption, we have a bounded loss $\map{\loss}{\hypoclass \times \samplespace}{[0, 1]}$ and therefore for all $f \in \hypoclass$
    \begin{align*}
        \gengap(f) = \left(\Expval_\RVZ [\loss(f, \cdot)] - \frac{1}{\samplesize} \sum_{i=1}^\samplesize \loss(f, z_i)\right) \in [-1,1]
    \end{align*}
    and in particular
    \begin{align*}
        \psi \coloneqq \constgengap^{-1} \gengap^2 \in [0, \constgengap^{-1}].
    \end{align*}
    Furthermore, since $\psi$ is a composition of measurable functions by definition, both $\psi$ and $-\psi$ are  measurable and bounded functions, i.e., they belong to $\bounmeasfunc{\hypoclass}$. Together with the assumption $\constgengap^{-1} \gengap^2 \oplus -\constgengap^{-1} \gengap^2 \leq \costfunc$ holding $\Prob_\RVsample$-a.s., Kantorovich's duality in Theorem~\ref{theorem: kantorovich duality} implies for $\Prob_\RVsample$ almost all $\sample$ and for all distributions $\posterior$ over $\hypoclass$
    \begin{align*}
        \posterior[\constgengap^{-1} \gengap^2] + \prior[-\constgengap^{-1} \gengap^2] \leq \opttransdist{\posterior}{\prior} \, .
    \end{align*}
    Equivalently,
    \begin{align*}
        \posterior[2(\samplesize-1) \gengap^2] + \prior[-2(\samplesize-1) \gengap^2] \leq \constgengap \cdot 2(\samplesize - 1) \cdot \opttransdist{\posterior}{\prior} \, .
    \end{align*}
    Therefore, rearranging and Jensen's inequality applied to the exponential function yield  
    \begin{align*}
        2(\samplesize-1) \posterior \left[\gengap^2\right] - \constgengap \cdot 2(\samplesize - 1) \cdot \opttransdist{\posterior}{\prior}
        \leq \prior \left[2(\samplesize-1) \gengap^2\right]
        \leq \ln{\prior \left[e^{2(\samplesize-1) \gengap^2}\right]} \, .
    \end{align*}Since it holds for all distributions $\posterior$ over $\hypoclass$, we get
    \begin{align*}
        f(\sample) = \sup_\posterior \left( 2(\samplesize-1) \posterior \left[\gengap^2 \right] - \constgengap \cdot 2(\samplesize - 1) \cdot \opttransdist{\posterior}{\prior} \right) \leq \ln{\prior \left[e^{2(\samplesize-1) \gengap^2}\right]}.
    \end{align*}
    By the independence of the prior $\prior$ and the sample set $\RVsample$, we get
    \begin{align*}
        \Expval_\RVsample\left[ e^f \right] \leq \Expval_\RVsample \left[ \prior \left[ e^{2(\samplesize-1) \gengap^2}\right] \right] = \prior \left[ \Expval_{\RVsample}\left[ e^{2(\samplesize-1) \gengap^2}\right]\right].
    \end{align*}
    Let $h \in \hypoclass$. Hoeffding's inequality yields 
    \begin{align*}
    \Prob \left[ \Delta_n(h) \geq \varepsilon \right] \leq e^{-2\samplesize\varepsilon^2},
    \end{align*}
    which implies by Lemma~5 in \citep{mcallester2003simplified}
    \begin{align*}
        \Expval_{\RVsample}\left[ e^{2(\samplesize-1) \gengap(h)^2}\right] \leq n.
    \end{align*}
    Hence $\Expval_\RVsample\left[ e^{f(\RVsample)}\right] \leq \samplesize$. Combining this with \eqref{eq: proof opt. trans., markov ineq.}, we get
    \begin{align*}
        \Prob[f(\RVsample) \geq \varepsilon ] \leq \frac{\samplesize}{e^\varepsilon}.
    \end{align*}
    Choosing $\varepsilon = \ln{\frac{\samplesize}{\delta}}$ and rearranging terms yields the statement.
\end{proof}

\section{Proof of Theorem~\ref{theorem: linear regression concave OT}}
\label{appendix: proof of theorem linear regression concave OT}

We first restate Theorem~\ref{theorem: linear regression concave OT} with all constants made explicit. We then establish the claimed bound by combining three intermediate lemmas.

\begin{theorem}
    Let $\map{\phi}{[0,\infty)}{[0,\infty)}$ be non-decreasing and concave, with $\phi(0)=0$ and $\phi(r)>0$ for every $r>0$. Let $\RVZ = (\RVX, \RVY)$ be a random variable taking values in $\R^d\times\R$. Assume that there exist $R, Y_{\max} > 0$ such that
    \begin{align*}
        \|\RVX\|_2\leq R \qquad\text{and}\qquad |\RVY|\leq Y_{\max}
    \end{align*}
    almost surely. Consider the hypothesis class and squared loss
    \begin{align*}
        \hypoclass &= \left\{ w\in\R^d:\| w\|_2\leq B\right\},\qquad \ell(w,(x,y)) = \frac{1}{(B R + Y_{\max})^2}\left(\langle w,x\rangle-y\right)^2,
    \end{align*}
    where $B>0$. Equip $\hypoclass$ with the Euclidean metric $d(w,v)=\| w-v\|_2$. Let $\delta\in(0,1)$, $\samplesize \geq 2$ and $\prior$ be a prior distribution over $\hypoclass$. Then, for every posterior distribution $\posterior$ over $\hypoclass$,
    \begin{multline*}
        \posterior[\risk] \leq \posterior[\emprisk(\cdot,\RVsample)] \\
        + \sqrt{ \frac{4B}{\phi(2B)}  \cdot \left( \frac{B^2R^2 + Y_{\max}^2}{2} + BR\RVY_{\max}\right) \left( BR^2+RY_{\max} \right) \frac{\left(1 + \sqrt{\ln \frac{12}{\delta}}\right)^2}{n} \phiOT{\prior}{\posterior}+ \frac{\ln\frac{2\samplesize}{\delta}}{2(\samplesize-1)}}
    \end{multline*}
    holds with probability at least $1-\delta$ over an i.i.d.\ sample $\RVsample=(\RVZ_1,\ldots,\RVZ_\samplesize)$.
\end{theorem}

\begin{assumption}[Uniform convergence \citep{shalev2014understanding}]
\label{assumption: UC}
    There exists a bound function $u(\samplesize, \delta) > 0$ such that for $\delta \in (0,1)$, $u(\samplesize, \delta) \xrightarrow{\samplesize \to \infty} 0$ and for any $\samplesize \in \N$
    \begin{align*}
        \Prob \left[ |\gengap(h)| \leq u(\samplesize, \delta), \ \forall h \in \hypoclass \right] \geq 1 - \delta.
    \end{align*}
\end{assumption}

\begin{assumption}[Loss-Gradient UC Property]
\label{assumption: UC grad}
    There exists a bound function $u_{\text{grad}}(\samplesize, \delta) > 0$ such that for any $\delta \in (0,1)$, $u_{\text{grad}}(\samplesize, \delta) \xrightarrow{\samplesize \to \infty} 0$ and for any $\samplesize \in \N$
    \begin{align*}
        \Prob\left[\sup_{f \neq g}\frac{|\gengap(f,\RVsample) - \gengap(g,\RVsample)|}{\phi(d(f,g))} \leq u_{\mathrm{grad}}(\samplesize,\delta)\right]\ge1-\delta.
    \end{align*}
\end{assumption}

\begin{remark}
\label{remark: uc grad}
    Assumption~\ref{assumption: UC grad} generalizes the Loss-Gradient Uniform Convergence (UC) property introduced by \cite{amit2022integral}, which requires that with probability at least $1 - \delta$
    \begin{align*}
        \sup_{h \in \operatorname{Int}(\hypoclass)} \left\| \Expval_\RVZ\big[\nabla_h \loss(h, Z)\big] - \frac{1}{\samplesize} \sum_{i=1}^{\samplesize} \nabla_h \loss(h, z_i) \right\|_2 \le u_{\text{grad}}(\samplesize, \delta),
    \end{align*}
    under the assumption that $\loss: \hypoclass \times \samplespace \to \R$ is differentiable w.r.t.\ $h$ on $\operatorname{Int}(\hypoclass) \times \samplespace$ and continuous on $\hypoclass \times \samplespace$.
\end{remark}

\begin{lemma}
    \label{lemma: OT PAC-Bayes bound uniform bounds}
    Let $(\hypoclass,d)$ be a metric space, $\map{\loss}{\hypoclass \times \samplespace}{[0,1]}$ be measurable, $\prior$ be a prior distribution over $\hypoclass$, $\delta\in(0,1)$, $\samplesize \geq 2$ and let $\map{\phi}{[0, \infty)}{[0, \infty)}$ be non-decreasing with $\phi(0)=0$ and $\phi(r) > 0$ for $r \neq 0$. Assume that the Assumptions~\ref{assumption: UC} and \ref{assumption: UC grad} are satisfied. Then, for every posterior distribution $\posterior$ over $\hypoclass$,
    \begin{align*}
        \posterior[\risk] \leq \posterior[\emprisk(\cdot,\RVsample)] + \sqrt{ 2 \cdot u\left(\samplesize, \frac{\delta}{4}\right) u_{\text{grad}}\left(\samplesize, \frac{\delta}{4}\right) \phiOT{\prior}{\posterior}+ \frac{\ln\frac{2\samplesize}{\delta}}{2(\samplesize-1)}}
    \end{align*}
    holds with probability at least $1-\delta$.
\end{lemma}

\begin{proof}
    The proof of this theorem is based on Theorem~\ref{theorem: optimal transport PAC-Bayes bound}, where the Lipschitz constant $\constgengap$ will be bounded using Assumptions~\ref{assumption: UC} and \ref{assumption: UC grad}. For this, let $f,g \in \hypoclass$, then
    \begin{align*}
        |\gengap(f)^2 - \gengap(g)^2| = |\gengap(f) - \gengap(g)| \cdot |\gengap(f) + \gengap(g)|.
    \end{align*}
    Assumptions~\ref{assumption: UC} provides, that
    \begin{align*}
        |\gengap(f) + \gengap(g)| \leq 2 \sup_{h \in \hypoclass} |\gengap(h)| \leq 2 \cdot u\left(\samplesize, \frac{\delta}{4}\right)
    \end{align*}
    holds with probability at least $1-\frac{\delta}{4}$. Furthermore, Assumption~\ref{assumption: UC grad} yields
    \begin{align*}
        |\gengap(f) - \gengap(g)| \leq u_{\text{grad}}\left(\samplesize, \frac{\delta}{4}\right) \cdot \phi(d(f,g))
    \end{align*}
    with probability at least $1-\frac{\delta}{4}$. Combining both, by the union bound argument,
    \begin{align*}
        |\gengap(f)^2 - \gengap(g)^2| \leq 2 \cdot u\left(\samplesize, \frac{\delta}{4}\right) u_{\text{grad}}\left(\samplesize, \frac{\delta}{4}\right) \cdot \phi(d(f,g))
    \end{align*}
    holds with probability at least $1-\frac{\delta}{2}$. Taking $\constgengap = 2 \cdot u\left(\samplesize, \frac{\delta}{4}\right) u_{\text{grad}}\left(\samplesize, \frac{\delta}{4}\right)$ in Theorem~\ref{theorem: optimal transport PAC-Bayes bound} and using again the union bound implies the statement.
\end{proof}

\begin{definition}[Bounded Differences Property]
    A function $\map{f}{\inspace_1 \times \inspace_2 \times \dots \times \inspace_\samplesize}{\R}$ satisfies the \textit{bounded differences property} if substituting the value of the $i$-th coordinate $x_i$ changes the value of $f$ by at most $c_i$. More precisely, if there are constants $c_1, c_2, \dots, c_\samplesize$ such that for all $i \in [\samplesize]$ and all $x_1 \in \inspace_1, x_2 \in \inspace_2, \dots, x_\samplesize \in \inspace_\samplesize$,
    \begin{align*}
        \sup_{x_i' \in \inspace_i} \left| f(x_1, \dots, x_{i-1}, x_i, x_{i+1}, \dots, x_\samplesize) - f(x_1, \dots, x_{i-1}, x_i', x_{i+1}, \dots, x_\samplesize) \right| \leq c_i.
    \end{align*}
\end{definition}

\begin{lemma}[McDiarmid's Inequality, \citep{mcdiarmid1989bounded}]
\label{lemma: McDiarmid's inequality}
    Let $\map{f}{\inspace_1 \times \inspace_2 \times \dots \times \inspace_\samplesize}{\R}$ satisfy the bounded differences property with bounds $c_1, c_2, \dots, c_\samplesize$. Consider independent random variables $\RVX_1, \RVX_2, \dots, \RVX_\samplesize$ where $\RVX_i \in \inspace_i$ for all $i$. Then, for any $\varepsilon > 0$,
    \begin{align*}
        \Prob \left[ f(\RVX_1, \RVX_2, \dots, \RVX_\samplesize) - \Expval\left[ f(\RVX_1, \RVX_2, \dots, \RVX_\samplesize) \right] \geq \varepsilon \right] &\leq \exp \left( -\frac{2\varepsilon^2}{\sum_{i=1}^\samplesize c_i^2} \right),\\
        \Prob \left[ f(\RVX_1, \RVX_2, \dots, \RVX_\samplesize) - \Expval\left[ f(\RVX_1, \RVX_2, \dots, \RVX_\samplesize) \right] \leq -\varepsilon \right] &\leq \exp \left( -\frac{2\varepsilon^2}{\sum_{i=1}^\samplesize c_i^2} \right),
    \end{align*}
    and as an immediate consequence,
    \begin{align*}
    \Prob \left[ \left| f(\RVX_1, \RVX_2, \dots, \RVX_\samplesize) - \Expval\left[ f(\RVX_1, \RVX_2, \dots, \RVX_\samplesize) \right] \right| \geq \varepsilon \right] \leq 2 \exp \left( -\frac{2\varepsilon^2}{\sum_{i=1}^\samplesize c_i^2} \right).
    \end{align*}
\end{lemma}

The following lemma provides a dimension-free Hoeffding-type concentration bound for empirical means of bounded Hilbert-space-valued random elements.

\begin{lemma}
\label{lemma: concentration bound}
    Let $K > 0$. Let $\RVX, \RVX_1,\ldots,\RVX_\samplesize$ be i.i.d. random vectors taking values in a Hilbert space $(\inspace, \langle \cdot , \cdot \rangle)$ satisfying $\|\RVX_i\|\leq K$ almost surely. Then with probability at least $1-\eta$,
    \begin{align*}
        \left\| \Expval[\RVX] - \frac{1}\samplesize\sum_{i=1}^\samplesize\RVX_i \right\| \leq \frac{K}{\sqrt \samplesize} \cdot \left( \sqrt{2\ln(1/\eta)} + 1 \right).
    \end{align*}
\end{lemma}

\begin{proof}
     Jensen's inequality and the i.i.d. assumption entail
    \begin{align*}
        \Expval \left[ \left\| \Expval[\RVX] - \frac{1}{\samplesize}\sum_{i=1}^\samplesize \RVX_i \right\|\right]^2 
        \leq \Expval \left[ \left\| \Expval[\RVX] - \frac{1}{\samplesize}\sum_{i=1}^\samplesize \RVX_i \right\|^2 \right]
        = \frac{1}{\samplesize} \Expval \left[ \left\| \Expval[\RVX] - \RVX \right\|^2 \right] = \frac{1}{\samplesize} \left (\Expval \left[ \left\|\RVX \right\|^2 \right] - \left\|\Expval[\RVX]\right\|^2 \right) .
    \end{align*}
    Therefore
    \begin{align}
    \label{eq: bound xi and expacval xi}
        \Expval \left[ \left\| \Expval[\RVX] - \frac{1}{\samplesize}\sum_{i=1}^\samplesize \RVX_i \right\|\right] 
        &\leq \left( \frac{1}{\samplesize} \Expval \left[\left\| \RVX-\Expval[\RVX] \right\|^2 \right] \right)^{1/2}
        \leq \frac{\Expval\left[\|\RVX\|^2\right]^\frac{1}{2}}{\sqrt{\samplesize}}
        \leq \frac{K}{\sqrt \samplesize}.
    \end{align}
    We apply McDiarmid's bounded-differences inequality from Lemma~\ref{lemma: McDiarmid's inequality} to $\map{f}{\inspace^\samplesize}{\R}$, where $f(x_1, \dots, x_\samplesize) \coloneqq \left\| \Expval[\RVX]- \frac{1}\samplesize\sum_{i=1}^\samplesize x_i \right\|$. Since $\|\RVX_i\|\leq K$ almost surely, the bounded differences property holds almost surely for $c_i \coloneqq \frac{2K}{\samplesize}$. McDiarmid's inequality then implies 
    \begin{align*}
        \Prob \left[f (\RVX_1, \dots, \RVX_\samplesize) - \Expval \left[f (\RVX_1, \dots, \RVX_\samplesize)\right] \geq \varepsilon \right] \leq 1 - \exp{\left(-\frac{\varepsilon^2 \samplesize}{2K^2}\right)} .
    \end{align*}
   Taking $\varepsilon \coloneqq K \sqrt \frac{2\ln{(1/\eta)}}{\samplesize}$ and using \eqref{eq: bound xi and expacval xi}, we obtain
    \begin{align*}
        \Prob \left[ \left\| \Expval[\RVX] - \frac{1}\samplesize\sum_{i=1}^\samplesize\RVX_i \right\| \geq \frac{K}{\sqrt \samplesize} \cdot \left( \sqrt{2\ln(1/\eta)} + 1 \right)\right] 
        \leq
        \Prob \left[f (\RVX_1, \dots, \RVX_\samplesize) \geq K \sqrt{\frac{2 \ln(1/\eta)}{\samplesize}} + \Expval \left[f (\RVX_1, \dots, \RVX_\samplesize)\right] \right] \leq \eta.
    \end{align*}
\end{proof}

\begin{lemma}[Bounded linear regression under a concave modulus]
    \label{lemma: linear-regression-concave-modulus}
    Let $\map{\phi}{[0,\infty)}{[0,\infty)}$ be non-decreasing and concave, with $\phi(0)=0$ and $\phi(r)>0$ for every $r>0$. Let $\RVZ = (\RVX, \RVY)$ be a random variable taking values in $\R^d\times\R$. Assume that there exist $R, Y_{\max} > 0$ such that
    \begin{align*}
        \|\RVX\|_2\leq R \qquad\text{and}\qquad |\RVY|\leq Y_{\max}
    \end{align*}
    almost surely. Consider the hypothesis class and squared loss
    \begin{align*}
        \hypoclass &= \left\{ w\in\R^d:\| w\|_2\leq B\right\},\qquad \ell(w,(x,y)) = \frac{1}{2}\left(\langle w,x\rangle-y\right)^2,
    \end{align*}
    where $B>0$. Equip $\hypoclass$ with the Euclidean metric $d(w,v)=\| w-v\|_2$. Let
    \begin{align*}
        \kappa_\samplesize(\eta) = \frac{1}{\sqrt n} + \sqrt{\frac{2\ln(1/\eta)}\samplesize}
    \end{align*}
    and $\delta\in(0,1)$. Then with probability at least $1-\delta$ over an i.i.d.\ sample $\RVsample=(\RVZ_1,\ldots,\RVZ_\samplesize)$, the following two inequalities hold simultaneously
    \begin{align*}
        \sup_{w\in\hypoclass}|\gengap(w)| \leq u(\samplesize,\delta)  \quad \text{and} \quad \sup_{w,v\in\hypoclass:\ w \neq v} \frac{|\gengap(w)-\gengap(v)|} {\phi(\| w-v\|_2)} \leq u_{\text{grad}}(\samplesize,\delta),
    \end{align*}
    where
    \begin{align*}
        u(\samplesize,\delta) &= \left( \frac{B^2R^2 + Y_{\max}^2}{2} + BR\RVY_{\max}\right) \kappa_\samplesize(\delta/3), \quad \text{and}\\
        u_{\text{grad}}(\samplesize,\delta) &=  \frac{2B}{\phi(2B)} \left( BR^2+RY_{\max} \right) \kappa_\samplesize(\delta/3).
    \end{align*}
\end{lemma}

\begin{proof}
    First, expanding the squared loss gives
    \begin{align*}
        \gengap(w) &= \risk(w) - \emprisk(w, \RVsample)
        = \Expval\left[ \frac{1}{2} \left(\langle w, \RVX\rangle - \RVY \right)^2 \right] - \frac{1}{\samplesize} \sum_{i=1}^\samplesize \frac{1}{2} \left(\langle w, \RVX_i\rangle - \RVY_i \right)^2\\
        &= \frac{1}{2} w \Expval[\RVX \RVX^\top] w^\top - w^\top\Expval[\RVX \RVY] + \frac{1}{2} \Expval[\RVY^2]
        - \sum_{i=1}^\samplesize\left( \frac{1}{2} w \Expval[\RVX_i \RVX_i^\top] w^\top - w^\top\Expval[\RVX_i \RVY_i] + \frac{1}{2}\Expval[\RVY_i^2] \right)\\
        &=\frac{1}{2}w^\top A_Sw + w^\top b_S + \frac{1}{2}c_S,
    \end{align*}
    where
    \begin{align*}
        A_S = \Expval[\RVX\RVX^\top] - \frac{1}{\samplesize}\sum_{i=1}^\samplesize \RVX_i \RVX_i^\top, \quad
        b_S = \Expval[\RVY\RVX] - \frac{1}{\samplesize}\sum_{i=1}^\samplesize \RVY_i\RVX_i, \quad \text{and} \quad 
        c_S = \Expval[\RVY^2] - \frac{1}{\samplesize}\sum_{i=1}^\samplesize \RVY_i^2.
    \end{align*}
    Next, we apply Lemma~\ref{lemma: concentration bound} to $\RVX\RVX^\top$, viewed as an element of the Hilbert space of matrices equipped with the Frobenius norm, to $\RVY\RVX\in\R^d$ and to $\RVY^2\in\R$, where
    \begin{align*}
            \| \RVX\RVX^\top\|_{F} = \|\RVX\|_2^2 \leq R^2, \qquad
            \| \RVY\RVX\|_2 \leq R\RVY_{\max}, \quad \text{and} \quad
            \RVY^2 \leq Y_{\max}^2,
    \end{align*}
    almost surely. Hence, a union bound shows that, with probability at least $1-\delta$,
    \begin{align*}
            \| A_S\|_{\mathrm{op}} \leq \| A_S\|_{F} \leq R^2\kappa_\samplesize(\delta/3),  \qquad
            \| b_S\|_2 \leq RY_{\max}\kappa_\samplesize(\delta/3),  \qquad
            |c_S| \leq Y_{\max}^2\kappa_\samplesize(\delta/3)
    \end{align*}
    hold simultaneously. Consequently, with probability at least $1-\delta$, for every $w\in\hypoclass$,
    \begin{align*}
        |\gengap(w)| &\leq \frac{1}{2}\| w\|_2^2 \| A_S\|_{\mathrm{op}} + \| w\|_2\| b_S\|_2 + \frac{1}{2}|c_S|
        \leq \left(\frac{B^2R^2}{2} + BRY_{\max} + \frac{Y_{\max}^2}{2}\right) \kappa_\samplesize(\delta/3),
    \end{align*}
    which proves the first inequality.
    
    For $w,v\in\hypoclass$, the symmetry of $A_S$ yields
    \begin{align*}
        \gengap(w)-\gengap(v) &= \frac{1}{2}w^\top A_Sw + w^\top b_S + \frac{1}{2}c_S - \left( \frac{1}{2}v^\top A_S v + v^\top b_S + \frac{1}{2}c_S \right)\\
        &= \frac{1}{2} (w-v)^\top A_S(w+v) + (w-v)^\top b_S.
    \end{align*}
    Hence, by the triangle inequality and the Cauchy Schwarz inequality
    \begin{align*}
        |\gengap(w)-\gengap(v)| &\leq \| w-v\|_2 \left(\frac{1}{2} \| A_S\|_{\mathrm{op}} \| w+v\|_2 + \| b_S\|_2\right)\\    &\leq \| w-v\|_2 \left( B\| A_S\|_{\mathrm{op}} + \| b_S\|_2\right)
    \end{align*}
    and therefore with probability at least $1 - \delta$
    \begin{align*}
        \frac{|\gengap(w)-\gengap(v)|}{\| w-v\|_2} \leq \left( BR^2+RY_{\max} \right) \kappa_\samplesize(\delta/3).
    \end{align*}
    It remains to compare the Euclidean distance with its modulus under $\phi$. Since $0 \leq \| w-v\|_2 \leq 2B$, let $t \in [0,2B]$. Every $t\in[0,2B]$ can be written as a convex combination of $0$ and $2B$. Concavity and $\phi(0)=0$ therefore imply
    \begin{align*}
        \phi(t) &= \phi\left( \frac{t}{2B}2B + \left(1-\frac{t}{2B}\right)0\right)
        \geq \frac{t}{2B}\phi(2B).
    \end{align*}
    Hence, for every $t\in(0,2B]$,
    \begin{align*}
        \frac{t}{\phi(t)} \leq \frac{2B}{\phi(2B)}.
    \end{align*}
    Combining the preceding inequalities gives
    \begin{align*}
        \sup_{ w,v\in\hypoclass: \ w \neq v} \frac{|\gengap(w)-\gengap(v)|} {\phi(\| w-v\|_2)}
        &\leq \sup_{ w,v\in\hypoclass: \ w \neq v} \frac{2B}{\phi(2B)} \frac{|\gengap(w)-\gengap(v)|} {\| w-v\|_2}\\
        &\leq \frac{2B}{\phi(2B)} \left(BR^2+RY_{\max}\right) \kappa_\samplesize(\delta/3)
    \end{align*}
    with probability at least $1-\delta$ which proves the second inequality.
\end{proof}

\begin{proof}[Proof of Theorem~\ref{theorem: linear regression concave OT}]
    Combining Lemma~\ref{lemma: linear-regression-concave-modulus} and Lemma~\ref{lemma: OT PAC-Bayes bound uniform bounds} proves Theorem~\ref{theorem: linear regression concave OT}
\end{proof}

\section{Proof of Theorem~\ref{theorem: rademacher complexity TV PAC-Bayes}}
\label{appendix: proof of theorem rademacher complexity TV PAC-Bayes}

For the proof of Theorem~\ref{theorem: rademacher complexity TV PAC-Bayes}, we first rederive the PAC-Bayesian bound for the TV distance from \citep{amit2022integral} through the lens of optimal transport. We then combine this bound with a standard result on Rademacher complexity from classical learning theory.

\begin{theorem}\citep[Theorem 6]{amit2022integral}
    \label{theorem: TV distance PAC-Bayes bound}
    Let $\map{\loss}{\hypoclass \times \samplespace}{[0, 1]}$ be a measurable loss function. Assume Assumption~\ref{assumption: UC} holds with bound function $\map{u}{\N \times (0,1)}{\R}$. Further, let $\prior$ be a prior distribution over $\hypoclass$, $\samplesize \geq 2$ and let $\delta \in (0,1)$. Then for any distribution $\posterior$ over $\hypoclass$
    \begin{align*}
        \posterior[\risk] \leq \posterior[\emprisk(\cdot, \RVsample)] + \sqrt{u^2\left(\samplesize, \frac{\delta}{2}\right) \ \TVdist{\posterior}{\prior} + \frac{\ln\frac{2\samplesize}{\delta}}{2(\samplesize-1)}}
    \end{align*}
    holds with probability at least $1-\delta$, where $\RVsample=\{\RVZ_i\}_{i=1}^\samplesize$ are $\samplesize$ i.i.d. copies of $\RVZ$.
\end{theorem}

\begin{proof}
    Let $n \in \N$ and $\delta \in (0,1)$ be fixed. Let $\prior$ and $\posterior$ be distributions over $\hypoclass$ and let $\map{h}{\hypoclass}{\R}$ be given by $h(f) \coloneqq 2 (\samplesize -1) \gengap(f)^2$. 
    The Kantorovich duality in Theorem~\ref{theorem: kantorovich duality} provides the following dual formulation for the TV distance
    \begin{align*}
        \TVdist{\posterior}{\prior} = \sup_{\substack{h_1, h_2 \in \bounmeasfunc{\hypoclass}: \\ h_1 \oplus h_2 \leq \mathbb{1}_{\cdot \neq \cdot}}} \posterior [h_1] + \prior[h_2] \,.
    \end{align*}
    Given $u(\samplesize, \delta) > 0$, it can be reformulated to
    \begin{align}
    \label{eq: reformulation of kantorovich duality TV}
        2(\samplesize-1)u^2(\samplesize, \delta) \cdot \wassersteinpdist{\posterior}{\prior}^p = \sup_{\substack{h_1, h_2 \in \bounmeasfunc{\hypoclass}: \\ h_1 \oplus h_2 \leq 2(\samplesize-1)u^2(\samplesize, \delta) \mathbb{1}_{\cdot \neq \cdot}}} \posterior [h_1] + \prior[h_2] \,.
    \end{align}
    To apply Theorem~\ref{theorem: optimal transport PAC-Bayes bound}, we need to verify that $h(f)-h(g) \leq 2(\samplesize-1) u^2(\samplesize, \delta) \mathbb{1}_{f \neq g}$ for all $f, g \in \hypoclass$. By Assumption~\ref{assumption: UC}
    \begin{align*}
        \Prob \left[ |\gengap(f)| \leq u(\samplesize, \delta), \ \forall f \in \hypoclass \right] \geq 1 - \delta.
    \end{align*}
    Since the loss function $\ell$ maps to the interval $[0,1]$, also the generalization gap 
    \begin{align*}
        |\gengap| = \left| \Expval \left[ \ell(\cdot, \RVZ) \right] - \frac{1}\samplesize \sum_{i=1}\samplesize \ell(\cdot, \RVZ_i) \right| \in [0,1].
    \end{align*}
    Therefore, also
    \begin{align}
    \label{eq: PAC squared gengap}
        \Prob \left[ \gengap^2(f) \leq u^2(\samplesize, \delta), \ \forall f \in \hypoclass \right] \geq 1 - \delta
    \end{align}
    holds. This implies $h(f)-h(g) \leq 2(\samplesize-1) u^2(\samplesize, \delta) \mathbb{1}_{f \neq g}$ for all $f, g \in \hypoclass$ with propbability at least $1-\delta$.
    
    The union bound of the statements in Theorem~\ref{theorem: optimal transport PAC-Bayes bound} and \eqref{eq: PAC squared gengap} both with confidence parameter $\frac{\delta}{2}$ implies the statement.
\end{proof}

The following standard uniform-convergence bound follows from symmetrization and McDiarmid's inequality, compare \citep[Theorem~8 and Appendix~B]{bartlett2002rademacher}.

\begin{proposition}[Rademacher uniform-convergence bound {\citep{bartlett2002rademacher}}]
    \label{proposition: Rademacher uniform convergence}
    Let
    \begin{align*}
        \map{\loss}{\hypoclass\times\samplespace}{[0,1]}
    \end{align*}
    be measurable and let $\mathcal{F}_{\loss}$ be the induced loss class from Definition~\ref{definition: expected Rademacher complexity}. Assume that the relevant suprema are measurable. Then, for every $\eta\in(0,1)$,
    \begin{align*}
        \Prob
        \left(
            \sup_{h\in\hypoclass}
            \left|
                \risk(h)
                -
                \emprisk(h,\RVsample)
            \right|
            \leq
            2\rademachercompl_{\samplesize}(\mathcal{F}_{\loss})
            +
            \sqrt{
                \frac{\ln(2/\eta)}
                     {2\samplesize}
            }
        \right)
        \geq
        1-\eta.
    \end{align*}
\end{proposition}

\begin{proof}[Proof of Theorem~\ref{theorem: rademacher complexity TV PAC-Bayes}]
    The statement follows directly from combining Proposition~\ref{proposition: Rademacher uniform convergence} and Theorem~\ref{theorem: TV distance PAC-Bayes bound} with $u(\samplesize, \delta) = 2\rademachercompl_{\samplesize}(\mathcal{F}_{\loss}) + \sqrt{ \frac{\ln(2/\eta)} {2\samplesize}}$.
\end{proof}

\section{Proof of Theorem~\ref{theorem: vc dimension PAC-Bayes}}
\label{appendix: proof of theorem VC dimension TV PAC-Bayes}

In this section, we proof Theorem~\ref{theorem: vc dimension PAC-Bayes}. The following proposition by \cite{boucheron2005theory} gives a well known uniform bound on the generalization gap at a rate of $O(\sqrt{\VCdim(\hypoclass)/\samplesize})$.

\begin{proposition}
\label{proposition: uniform bound by boucheron}
    Consider a binary classification problem, with the 0-1 loss, and hypothesis class $\hypoclass$, with finite VC dimension. Then there exists some constant $c > 0$, s.t. for any $\delta \in (0,1)$ we have
    \begin{align*}
        \Prob \left[ \gengap(h) \leq c \sqrt{\frac{\VCdim(\hypoclass) + \ln\frac{1}{\delta}}{\samplesize}}, \; \forall h \in \hypoclass \right] \geq 1 - \delta.
    \end{align*}
\end{proposition}

\begin{proof}[Proof of Theorem~\ref{theorem: vc dimension PAC-Bayes}]
    By Proposition~\ref{proposition: uniform bound by boucheron} is the uniform convergence assumption satisfied for $u(\samplesize, \delta) = c \sqrt{\frac{\VCdim(\hypoclass) + \ln\frac{1}{\delta}}{\samplesize}}$. Therefore, Theorem~\ref{theorem: TV distance PAC-Bayes bound} implies the statement.
\end{proof}

\section{Bounding the regularity constant using classical learning theory for general OT distances}
\label{appendix: bounding the regularity constant using classical learning theory for general OT distances}

In the main part of the paper, when using classical learning theory to control the regularity constant, we focused on the TV distance because it provides a particularly simple and compact special case. Classical learning-theoretic approaches typically control the generalization gap uniformly over the hypothesis class and therefore provide a bound for the uniform convergence quantity in Assumption~\ref{assumption: UC}. They do not, in general, directly provide bounds for the loss-gradient UC quantity in Assumption~\ref{assumption: UC grad}. The TV distance is special in this respect because its corresponding loss-gradient UC condition can be reduced to the standard UC condition. Since the TV distance is obtained for $\phi(r) = \indicatorfunc{r > 0}$, we have
\begin{align*}
    \sup_{f\neq g} \frac{|\gengap(f,\RVsample) - \gengap(g,\RVsample)|}{\phi(d(f,g))}
    &= \sup_{f\neq g} \frac{|\gengap(f,\RVsample) - \gengap(g,\RVsample)|}{\indicatorfunc{f \neq g}} 
    = \sup_{f\neq g}  |\gengap(f,\RVsample) - \gengap(g,\RVsample)| \\
    &\leq \sup_{f\neq g} |\gengap(f,\RVsample)| + |\gengap(g,\RVsample)| 
    \leq 2 \sup_{f} |\gengap(f,\RVsample)|.
\end{align*}
Thus, Assumption~\ref{assumption: UC} immediately implies Assumption~\ref{assumption: UC grad} for the TV distance, with $u_{\mathrm{grad}}(\samplesize,\delta) = 2u(\samplesize,\delta)$.

Consequently, the regularity constant in Lemma~\ref{lemma: OT PAC-Bayes bound uniform bounds} for general $\phi$-OT distances is bounded by the product of the two quantities $u$ and $u_{\text{grad}}$ from Assumptions~\ref{assumption: UC} and~\ref{assumption: UC grad}, whereas for the TV distance, Theorem~\ref{theorem: TV distance PAC-Bayes bound} yields a bound involving only $u^2$. Hence, standard uniform-convergence arguments can be applied directly to control the regularity constant in the TV case.

For general OT distances, however, controlling the generalization gap alone is not sufficient. In addition to controlling the magnitude of the generalization gap uniformly over the hypothesis class, one must control how rapidly the generalization gap can change between hypotheses. In our framework, this additional requirement is precisely captured by the loss-gradient UC condition in Assumption~\ref{assumption: UC grad}. Standard uniform-convergence results generally provide no such bound, since they control the pointwise magnitude of the generalization gap but not its variation across hypotheses. This is the reason why the classical learning-theoretic arguments used for the TV case do not directly extend to arbitrary OT distances.

The extension to more general OT distances is nevertheless possible under additional regularity assumptions. For example, if $\loss$ is differentiable with respect to the hypothesis and the corresponding class of loss gradients has finite VC dimension under suitable assumptions, one can apply uniform-convergence arguments separately to the gradient class. This yields a high-probability bound on the quantity in Assumption~\ref{assumption: UC grad}, which can then be combined with the standard UC bound from Assumption~\ref{assumption: UC} to control the regularity constant in Lemma~\ref{lemma: OT PAC-Bayes bound uniform bounds}. Thus, the limitation is not specific to OT distances themselves, but rather reflects the additional regularity information required. Whereas the TV case only requires uniform control of the generalization gap, general OT distances require both uniform control of its magnitude and uniform control of its variation across the hypothesis class.

\section{Proof of Theorem~\ref{theorem: finite p-Wasserstein distance PAC-Bayes bound}}
\label{appendix: proof of theorem finite p-Wasserstein distance PAC-Bayes bound}

In this section, we prove Theorem~\ref{theorem: finite p-Wasserstein distance PAC-Bayes bound}. We work under the assumptions of Theorem~\ref{theorem: finite p-Wasserstein distance PAC-Bayes bound}. For $h\ne h'$, let
\[
g_{h,h'}(z)=\frac{\ell(h,z)-\ell(h',z)}{d(h,h')^p},\qquad
\losshypoclass=\{g_{h,h'}:\ h,h'\in \hypoclass,\ h\neq h'\}.
\]
Denote 
\[
P f := \Expval[f(Z)] \quad \text{and} \quad
P_n f := \frac{1}{n}\sum_{i=1}^n f(Z_i).
\]
By boundedness of $\hypoclass$, $\|g\|_\infty\leq G$ for all $g\in\losshypoclass$, so $Pg$ is well defined.
Denote
\[
S_n := \sup_{g\in\losshypoclass}|(P_n-P)g|.
\]
The problem boils down to a standard bounded empirical-process supremum problem. For a probability measure $Q$ on $\mathcal Z$ with finite support, let
$N(\varepsilon,\losshypoclass,L_2(Q))$ be the covering number of $\losshypoclass$ with $L_2(Q)$-balls of radius $\varepsilon$.
Define
\[
J(1,\losshypoclass):=\sup_{Q}\int_0^1\sqrt{1+\ln \epscovnumb\big(\varepsilon L,\losshypoclass,L_2(Q)\big)}\;d\varepsilon,
\]
the supremum being over all finitely supported $Q$ on $\mathcal Z$. 

Note that $J(1,\losshypoclass) \geq 1$. On the other hand, the uniform H\"older continuity assumption of $\hypoclass$ does not imply $J(1,\losshypoclass)<\infty$. H\"older continuity gives $\|g\|_\infty\leq G$ but no control of the complexity of $\losshypoclass$ near the diagonal $h'\to h$.

\begin{lemma}\label{lemma: compact}
Under the assumptions of Theorem~\ref{theorem: finite p-Wasserstein distance PAC-Bayes bound}, there is a universal constant $c$ such that for all $n\ge1$ and $\delta\in(0,1)$, with probability at least $1-\delta$,
\[
\sup_{h\neq h'}\frac{|\gengap(h)-\gengap(h')|}{d(h,h')^p} \leq \frac{c\,G\,J(1,\losshypoclass)}{\sqrt n} + G\sqrt{\frac{2\ln(1 \delta)}{n}} .
\]
In particular the bound is $O(n^{-1/2})$ whenever $J(1,\losshypoclass)<\infty$.
\end{lemma}

\begin{proof}
Since $\hypoclass$ is compact, it has a countable dense subset $\hypoclass_0$. Let $\losshypoclass_0=\{g_{h,h'}:h,h'\in \hypoclass_0,\ h\ne h'\}$. Given $h\ne h'$, choose $h_m,h_m'\in \hypoclass_0$ with $h_m\to h$, $h_m'\to h'$, $h_m\neq h_m'$.
By uniform H\"older continuity, $g_{h_m,h_m'}(z)\to g_{h,h'}(z)$ for every $z$. Hence $\losshypoclass$ is pointwise measurable, in the sense that every element is a pointwise limit of a sequence from the countable class $\losshypoclass_0$. In particular $\losshypoclass$ is $P$-measurable in the sense of \cite[Example~2.3.4]{vandervaart98}, and $S_n$ is measurable.

$\losshypoclass$ is a $P$-measurable class with constant envelope $G$. 
It then follows from \cite[Theorem~2.14.1]{vanderVaartWellner96} that
\[
\Expval[S_n] \leq \frac{c\,G\,J(1,\losshypoclass)}{\sqrt n}.
\]
for a universal constant $c$. 

Now, replacing one $Z_i$ by $Z_i'$ changes each $(P_n-P)g$ by at most $|g(Z_i)-g(Z_i')|/n\leq 2G/n$, hence changes $S_n$ by at most $2G/n$. McDiarmid's inequality then gives
\[
\Prob(S_n-\Expval[S_n] \geq t) \leq \exp\left(-\frac{nt^2}{2G^2}\right),
\]
Taking $t=G\sqrt{2\ln(1/\delta)/n}$ yields, with probability at least $1-\delta$,
\[
S_n \leq \Expval[S_n]+G\sqrt{2\ln(1/\delta)/n} \leq \frac{c\,G\,J(1,\losshypoclass)}{\sqrt n} +G\sqrt{\frac{2\ln(1/\delta)}{n}}
\]
as claimed.
\end{proof}

\begin{proof}[Proof of Theorem~\ref{theorem: finite p-Wasserstein distance PAC-Bayes bound}]
    Combining Theorem~\ref{theorem: optimal transport PAC-Bayes bound} and Lemma~\ref{lemma: compact} yields the statement.
\end{proof}

\paragraph{Case of finite $\hypoclass$}
In this case, the elements of $\losshypoclass$ form a cover at every scale, and $|\losshypoclass| \leq |\hypoclass|(\hypoclass-1) \leq |\hypoclass|^2$. The supremum is over a finite set of at most $|\losshypoclass|$ unordered pairs and $S_n$ is a maximum of finitely many measurable functions. It then follows from Lemma~\ref{lemma: compact} that
\[
\sup_{h\neq h'}\frac{|\gengap(h)-\gengap(h')|}{d(h,h')^p} \leq G\sqrt{\frac{2\big(2\ln |\hypoclass| +\ln(1/\delta)\big)}{n}}\;
\]
with probability at least $1-\delta$.

\paragraph{In terms of VC-class index} Let $V$ be the VC-index of $\losshypoclass$ (the VC-index of the class of subgraphs). It follows from \cite[Theorem~2.6.7]{vanderVaartWellner96} that $\sup_Q \ln \epscovnumb(\varepsilon L,\losshypoclass,L_2(Q)) \leq cV+2(V-1)\ln(1/\varepsilon)$ for a universal constant $c$. Using $\sqrt{a+b}\le\sqrt a+\sqrt b$ and $\int_0^1\sqrt{\ln(1/\varepsilon)}\,d\varepsilon=\sqrt\pi/2\le 1$, we get
\[
J(1,\losshypoclass)\le\sqrt{1+cV}+\sqrt{2V}\lesssim\sqrt V .
\]

\paragraph{Case of $\ell(\cdot,z)$ is $C^{1,\alpha}$}
Let $\hypoclass \subset \R^k$ be compact and convex with $D=\operatorname{diam}(\hypoclass)$, $d(h,h')=|h-h'|$, and $p=1$. Assume in addition that $\ell(\cdot,z)$ is $C^1$ on a neighbourhood of $\hypoclass$ for each $z$, and its gradient is $\alpha$-H\"older continuous with modulus $M$ for all $z$. Set
\[
\kappa=1+\frac1\alpha,\qquad \rho=D\Big(\frac MG\Big)^{1/\alpha},\qquad
\Lambda=2^{\kappa}\big(16+32\cdot 8^{1/\alpha}\rho\big)\ (\ge1).
\]

Then with probability at least $1-\delta$,
\[
\sup_{h\neq h'}\frac{|\gengap(h)-\gengap(h')|}{|h-h'|} \leq
\frac{C G}{\sqrt n}\Big[\sqrt{1+(2k+1)\ln\Lambda}+\sqrt{(2k+1)\kappa}\Big]+G\sqrt{\frac{2\ln(1/\delta)}{n}} .
\]
Note that this scales as
\[
\sup_{h\neq h'}\frac{|\gengap(h)-\gengap(h')|}{|h-h'|} \lesssim
\sqrt{\frac{k}{n}}+\sqrt{\frac{\ln(1/\delta)}{n}} .
\]

\begin{proof}
In view of Lemma~\ref{lemma: compact}, it is sufficient to bound $J(1,\losshypoclass)$. 

Let $\Theta=\{(h,r,w): h\in \hypoclass,\ r\in(0,D],\ |w|=1,\ h-rw\in \hypoclass\}\subset\R^k\times\R\times\R^k$.
For $h\neq h'$ put $r=|h-h'|$, $w=(h-h')/r$. By convexity the segment $[h',h]$ lies in $\hypoclass$, and the fundamental theorem of calculus gives
\[
g_{h,h'}(z)=\varphi(\theta,z):=\int_0^1\langle\nabla\ell(h-s\,rw,z),w\rangle \,ds ,\qquad \theta=(h,r,w)\in\Theta .
\]
Hence $\losshypoclass=\{\varphi(\theta,\cdot):\theta\in\Theta\}$. For $\theta,\theta'\in\Theta$ and all $z$, by assumption,
\[
|\varphi(\theta,z)-\varphi(\theta',z)|\le M\sup_s|(h-srw)-(h'-sr'w')|^\alpha+L|w-w'|.
\]
Since $|(h-srw)-(h'-sr'w')|\leq |h-h'|+|r-r'|+D|w-w'|\leq (2+D)|\theta-\theta'|$, the first term is at most $M(2+D)^\alpha|\theta-\theta'|^\alpha$.
Since $|w-w'|\leq 2$, we have $|w-w'|\leq 2^{1-\alpha}|w-w'|^\alpha\leq 2|\theta-\theta'|^\alpha$. Therefore
\[
\sup_z|\varphi(\theta,z)-\varphi(\theta',z)|\leq K|\theta-\theta'|^\alpha=:\rho(\theta,\theta').
\]
Because $\alpha\in(0,1]$, $\rho$ is a metric on $\Theta$ (a concave power of a metric is a metric).

We now apply \cite[Theorem~2.7.11]{vanderVaartWellner96} with $T=\Theta$, $f_\theta=\varphi(\theta,\cdot)$, the metric $\rho$ above, envelope $F\equiv1$, and $\|\cdot\|=\|\cdot\|_\infty$. It gives
\[
\epscovnumb_{[\,]}(2\varepsilon,\losshypoclass,\|\cdot\|_\infty)\le \epscovnumb(\varepsilon,\Theta,\rho)=\epscovnumb\big(\Theta,|\cdot|,(\varepsilon/K)^{1/\alpha}\big),\qquad \varepsilon>0 .
\]
Here the bracketing number $\epscovnumb_{[\,]}(\varepsilon,\losshypoclass,\|\cdot\|_\infty)$ is the minimal number of brackets $[l, \bar{l}] = \{f:\mathcal{Z} \to \R: l \leq f \leq \bar{l}\}$ of size $\|\bar{l}-l\|_{\infty}$ smaller than $\varepsilon$ needed to cover $\losshypoclass$. A bracket $[l,\bar l]$ with $\|\bar l-l\|_\infty\le2\varepsilon$ is contained in the sup-norm ball of radius $\varepsilon$ around $(l+\bar l)/2$, so
\[
\epscovnumb(\varepsilon,\losshypoclass,\|\cdot\|_\infty)\leq \epscovnumb_{[\,]}(2\varepsilon,\losshypoclass,\|\cdot\|_\infty)\le \epscovnumb\big(\Theta,|\cdot|,(\varepsilon/K)^{1/\alpha}\big).
\]
The set $\Theta\subset\R^{2k+1}$ has diameter at most $2\Delta=2(D+1)$. By the volumetric bound (e.g. \cite[Corollary~4.2.13]{vershynin2018}, applied to the ball of radius $\Delta$ around a point of $\Theta$), 
\[
\epscovnumb(\varepsilon,\losshypoclass,\|\cdot\|_\infty)\le\Big(1+2\Delta\,(K/u)^{1/\alpha}\Big)^{2k+1}.
\]
Sup-norm balls are contained in $L_2(Q)$-balls of the same radius, so the same bound holds for $N(\varepsilon,\losshypoclass,L_2(Q))$ for every finitely supported $Q$.

Replace $\varepsilon$ with $\varepsilon G$ with $\varepsilon\in(0,1]$ above. Then $1+2\Delta(K/(\varepsilon G))^{1/\alpha}\le\Lambda\,\varepsilon^{-1/\alpha}$, so
\[
\ln \epscovnumb(\varepsilon L,\losshypoclass,L_2(Q)) \leq (2k+1)\big(\ln\Lambda+\frac{1}{\alpha}\ln(1/\varepsilon)\big).
\]
Using $\sqrt{a+b}\leq \sqrt a+\sqrt b$ and $\int_0^1\sqrt{\ln(1/\varepsilon)}\,d\varepsilon=\sqrt\pi/2\leq 1$,
\[
J(1,\losshypoclass)\le\sqrt{1+(2k+1)\ln\Lambda}+\sqrt{(2k+1)/\alpha}<\infty .
\]
\end{proof}

For finite dimensions, we have the following special case of the theorem.

\begin{theorem}
\label{theorem: compact p-Wasserstein distance PAC-Bayes bound}
    Let $\map{\loss}{\hypoclass \times \samplespace}{[0, 1]}$ be a measurable loss function and $\hypoclass \sub \R^d$ be a compact hypothesis metric space. Assume that $\loss(\cdot, z)$ is a $p$-Hölder continuous function on $\hypoclass$ w.r.t.\ the Euclidean norm $\| \cdot \|_2$ uniformly in $z \in \samplespace$, and denote $G$ its uniform Hölder continuity constant. Let $\prior$ be a prior distribution over $\hypoclass$, $\samplesize \geq 2$ and let $\delta \in (0,1)$. For $\varepsilon > 0$, let $\epscovnumb(\hypoclass, \|\cdot\|_2,\varepsilon)$ denote the $\varepsilon$-covering number of $\hypoclass$ w.r.t.\ $\|\cdot\|_2$. Then for any $p \geq 1$
    \begin{multline}
        \posterior[\risk] \leq \posterior[\emprisk(\cdot, \RVsample)] \\
        + \left( 4G  \left (\sqrt{\frac{\ln(2/\delta)+d \ln(3 \operatorname{diam}(\hypoclass)) + \frac{d}{2p} \ln\left(G^2\samplesize \right) }{2\samplesize}} + \frac{2}{\sqrt{\samplesize}} \right) \wassersteinpdist{\posterior}{\prior}^p 
        + \frac{\ln(2\samplesize/\delta)}{2(\samplesize - 1)}\right)^\frac{1}{2}
    \end{multline}
    holds with probability at least $1-\delta$ over $\RVsample$ simultaneously for any distribution $\posterior$ over $\hypoclass$.
\end{theorem}

\begin{lemma}
\label{lemma: compact, bound hölder const}
    Let $\hypoclass \sub \R$ be a compact hypothesis class and $\delta \in (0, 1)$. Assume that for any fixed $z \in \samplespace$, $|\loss(h, z) - \loss(h', z)| \leq G d(h, h')^p$. For $\varepsilon > 0$, let $\epscovnumb(\hypoclass, \| \cdot \|_2,\varepsilon)$ denote the $\varepsilon$-covering number of $\hypoclass$ w.r.t.\ $\|\cdot\|_2$. Then, with probability at least $1-\delta$ over $\RVsample=\{\RVZ_i\}_{i=1}^\samplesize$,
    \begin{align*}
        |\gengap(f)^2-\gengap(g)^2| \leq \constgengap(\varepsilon) d(f,g)^p
    \end{align*}
    holds for all $f,g\in\hypoclass$, where
    \begin{align*}
        \constgengap(\varepsilon) = 4G  \left (\sqrt{\frac{\ln(2/\delta)+d \ln(3 \operatorname{diam}(\hypoclass)) + \frac{d}{2p} \ln\left(G^2\samplesize \right) }{2\samplesize}} + \frac{2}{\sqrt{\samplesize}} \right).
    \end{align*}
\end{lemma}

\begin{proof}
    By the definition of the generalization gap, 
    \begin{align}
        \constgengap &= \sup_{h, h' \in \hypoclass : h \ne h'} \frac{\left|\gengap^2(h) - \gengap^2(h')\right|}{d(h, h')^p} \nonumber \\
        &= \sup_{h, h' \in \hypoclass : h \ne h'} \frac{|\gengap(h) - \gengap(h')| \, |\gengap(h) + \gengap(h')|}{d(h, h')^p} \nonumber \\
        &\leq \left( \sup_{h, h' \in \hypoclass : h \neq h'} \frac{|\gengap(h) - \gengap(h')|}{d(h, h')^p} \right) 2 \sup_{h \in \hypoclass} |\gengap(h)| \nonumber \\
        &= 2 \left( \sup_{h, h' \in \hypoclass : h \neq h'} \frac{\left|\left(\risk(h)-\emprisk(h, \RVsample)\right) - \left(\risk(h') - \emprisk(h', \RVsample)\right)\right|}{d(h, h')^p} \right) \sup_{h \in \hypoclass} |\gengap(h)| \nonumber \\
        &= 2 \left( \sup_{h, h' \in \hypoclass : h \neq h'} \frac{\left|\Expval \left[\loss(h, \RVZ) - \loss(h', \RVZ) \right] - \frac{1}\samplesize\sum_{i=1}^\samplesize \left( \loss(h,  \RVZ_i) - \loss(h',  \RVZ_i)\right) \right|}{d(h, h')^p} \right) \sup_{h \in \hypoclass} |\gengap(h)|. \label{eq: rewrite hölder const}
    \end{align}

For the rest of the proof, we will derive bounds for both $\sup_{h \in \hypoclass} |\gengap(h)|$ and $|\Expval \left[\loss(h, \RVZ) - \loss(h', \RVZ) \right] - \frac{1}\samplesize\sum_{i=1}^\samplesize \left( \loss(h,  \RVZ_i) - \loss(h',  \RVZ_i)\right) | / d(h, h')^p$ individually. The statement follows then by combining both bounds using a union bound argument. We first derive a uniform high probability bound for $\sup_{h \in \hypoclass} |\gengap(h)|$ using a covering argument.

Let $\varepsilon > 0$, and let $\hypoclass_\varepsilon = \{h_1,\dots,h_{\epscovnumb_\varepsilon}\}$ be an $\varepsilon$-covering of $\hypoclass$, where $\epscovnumb_\varepsilon = \epscovnumb(\hypoclass,\|\cdot\|_2,\varepsilon)$ is the $\varepsilon$-covering number. For every $h \in \hypoclass$, choose $h_\varepsilon \in \hypoclass_\varepsilon$ such that $d(h,h_\varepsilon)\leq \varepsilon$. By the Hölder assumption,
\begin{align*}
    |\risk(h)-\risk(h_\varepsilon)|
    &= \left |\Expval \left[\loss(h, \RVZ) - \loss(h_\varepsilon, \RVZ) \right] \right|
    \leq G\varepsilon^p, \nonumber \\
    |\emprisk(h,\RVsample)-\emprisk(h_\varepsilon,\RVsample)|
    &= \left| \frac{1}\samplesize\sum_{i=1}^\samplesize \left( \loss(h,  \RVZ_i) - \loss(h',  \RVZ_i)\right) \right| 
    \leq G\varepsilon^p.
\end{align*}
Consequently,
\begin{align}
    |\gengap(h)-\gengap(h_\varepsilon)|
    \leq 2G\varepsilon^p.
    \label{eq: compact approximation gengap}
\end{align}

Since $\hypoclass$ is compact, $\epscovnumb_\varepsilon<\infty$. For each fixed $h_j\in\hypoclass_\varepsilon$, Hoeffding's inequality yields
\begin{align*}
    \Prob\left( |\gengap(h_j)| \leq \sqrt{\frac{\ln(2\epscovnumb_\varepsilon/\delta)}{2\samplesize}} \right)
    \geq 1 - \frac{\delta}{\epscovnumb_\varepsilon}.
\end{align*}
A union bound over the finite $\varepsilon$-covering therefore yields
\begin{align}
\label{eq: uniform gengap bound compact}
    \Prob\left( \sup_{h_j \in \hypoclass_\varepsilon} |\gengap(h_j)| \leq \sqrt{\frac{\ln(2\epscovnumb_\varepsilon/\delta)}{2\samplesize}} \right)
    \geq 1-\delta.
\end{align}
Combining \eqref{eq: compact approximation gengap} and
\eqref{eq: uniform gengap bound compact}, we obtain
\begin{align}
    \sup_{h\in\hypoclass}|\gengap(h)|
    \leq \sup_{h\in\hypoclass,h_\varepsilon \in \hypoclass_\varepsilon} |\gengap(h_\varepsilon)| + |\gengap(h_\varepsilon) - \gengap(h)|
    \leq \sqrt{\frac{\ln(2\epscovnumb_\varepsilon/\delta)}{2\samplesize}} + 2G\varepsilon^p
    \label{eq: uniform gengap bound compact full}
\end{align}
with probability at least $1-\delta$.
By boundedness of $\hypoclass$, $\operatorname{diam}(\hypoclass) < \infty$. Therefore the volumetric bound in \citep[Corollary 4.2.13]{vershynin2018} yields
\begin{align*}
    \epscovnumb_\varepsilon \leq \left( 1 + \frac{2 \operatorname{diam}(\hypoclass)}{\epsilon} \right)^d
\end{align*}
and \eqref{eq: uniform gengap bound compact full} becomes
\begin{align*}
    \sup_{h\in\hypoclass}|\gengap(h)|
    \leq \sqrt{\frac{\ln\left(2\left( 1 + \frac{2 \operatorname{diam}(\hypoclass)}{\epsilon} \right)^d/\delta\right)}{2\samplesize}} + 2G\varepsilon^p
\end{align*}
For \(\varepsilon\leq D\),
\begin{align*}
    \sqrt{\frac{\ln\left(2\left( 1 + \frac{2 \operatorname{diam}(\hypoclass)}{\epsilon} \right)^d/\delta\right)}{2\samplesize}} + 2G\varepsilon^p
    \leq \sqrt{\frac{\ln(2/\delta) + d \ln\left( \frac{3 \operatorname{diam}(\hypoclass)}{\epsilon} \right)}{2\samplesize}} + 2G\varepsilon^p
\end{align*}
Choosing
\begin{align*}
    \varepsilon = \left(\frac{1} {G^2\samplesize}\right)^{1/(2p)}
\end{align*}
and substituting yields
\begin{align}
\label{eq: compact hölder const bound part 1}
    \sup_{h\in\hypoclass}|\gengap(h)| \leq \sqrt{\frac{\ln(2/\delta)+d \ln(3 \operatorname{diam}(\hypoclass)) + \frac{d}{2p} \ln\left(G^2\samplesize\right) }{2\samplesize}} + \frac{2} {\sqrt{\samplesize}}.
\end{align}
For the second term in \eqref{eq: rewrite hölder const}, let
$h,h'\in\hypoclass$ be arbitrary. By the Hölder assumption,
\begin{align*}
    \left| \Expval \left[\loss(h, \RVZ) - \loss(h', \RVZ) \right] \right|
    &\leq Gd(h,h')^p, \\
    \left| \frac{1}\samplesize \sum_{i=1}^\samplesize \left( \loss(h,\RVZ_i)-\loss(h',\RVZ_i) \right) \right|
    &\leq Gd(h,h')^p.
\end{align*}
Hence, 
\begin{align}
    &\sup_{h,h'\in\hypoclass:h\neq h'} \frac{ \left| \Expval \left[\loss(h, \RVZ) -  \loss(h', \RVZ) \right] - \frac{1}\samplesize \sum_{i=1}^\samplesize \left( \loss(h,\RVZ_i)-\loss(h',\RVZ_i) \right) \right|}{d(h,h')^p} \nonumber \\
    &\leq\sup_{h,h'\in\hypoclass:h\neq h'} \frac{ \left| \Expval \left[\loss(h, \RVZ) -  \loss(h', \RVZ) \right] \right| + \left|\frac{1}\samplesize \sum_{i=1}^\samplesize \left( \loss(h,\RVZ_i)-\loss(h',\RVZ_i) \right) \right|}{d(h,h')^p}
    \leq 2G.
    \label{eq: deterministic hoeffding bound compact}
\end{align}
Combining \eqref{eq: rewrite hölder const},
\eqref{eq: compact hölder const bound part 1}, and
\eqref{eq: deterministic hoeffding bound compact} yields the statement.
\end{proof}

\begin{proof}[Proof of Theorem~\ref{theorem: finite p-Wasserstein distance PAC-Bayes bound}]
    Applying a union bound to combine Theorem~\ref{theorem: optimal transport PAC-Bayes bound} for $\phi (r) = r^p$ and Lemma~\ref{lemma: compact, bound hölder const}, each evaluated with confidence parameter $\frac{\delta}{2}$, yields the statement.
\end{proof}

\section{Properties of the Statistical Distances}
\label{appendix: properties of the statistical distances}

\begin{figure}
    \centering
    \includegraphics[width=0.7\linewidth]{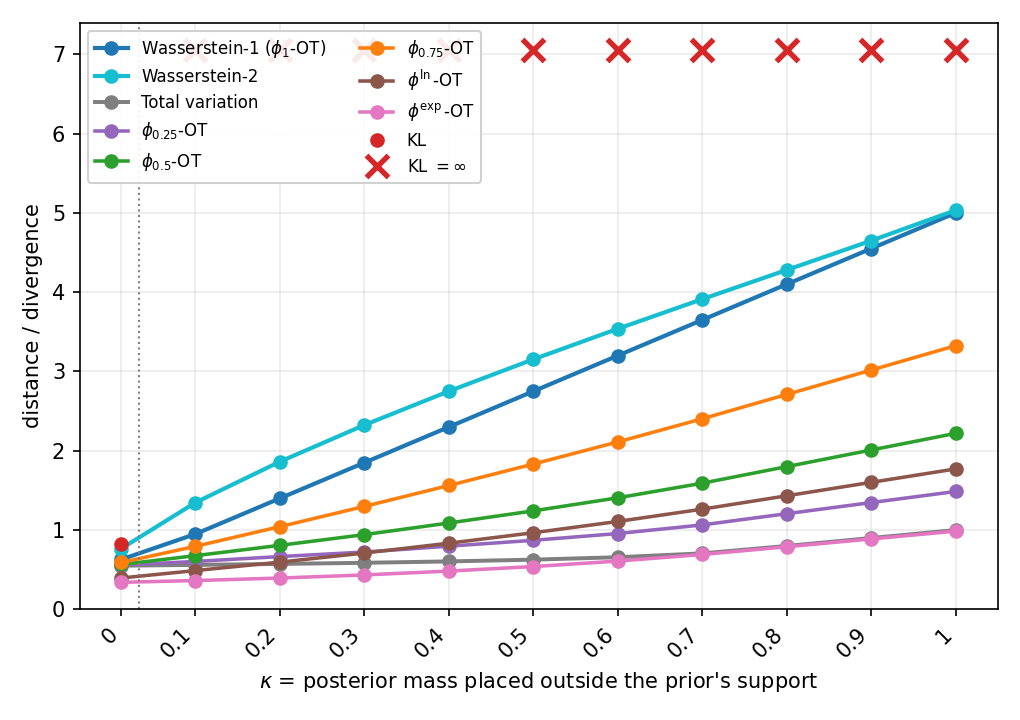}
    \caption{The posterior places a fraction $\kappa$ of its mass outside the support of the prior. The Wasserstein distance remains finite and varies continuously with $\kappa$, whereas the KL divergence is infinite for every $\kappa>0$.}
    \label{fig: support missmatch}
\end{figure}

\begin{figure}
    \centering
    \includegraphics[width=0.7\linewidth]{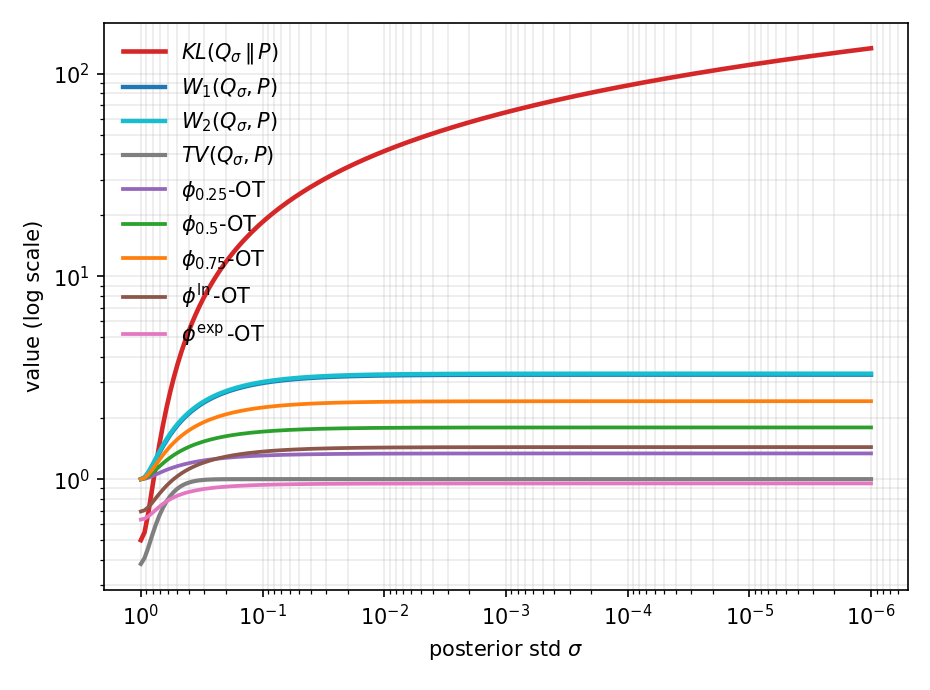}
    \caption{The posterior variance $\sigma^2$ decreases toward zero while the prior remains fixed. The KL divergence diverges as the posterior collapses, whereas the Wasserstein distance converges to a finite value.}
    \label{fig: gaussian collapse}
\end{figure}

\begin{figure}
    \centering
    \includegraphics[width=0.7\linewidth]{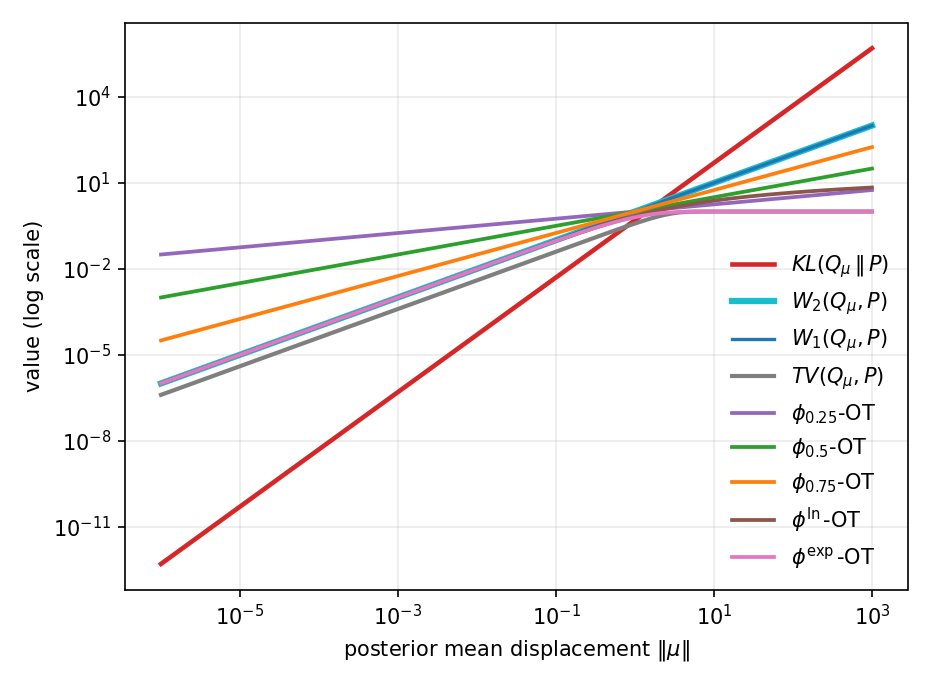}
    \caption{The posterior mean is shifted by $\mu$ while the prior and posterior covariances remain fixed. The KL divergence grows quadratically in $|\mu|_2$, whereas the $W_2$ distance grows linearly.}
    \label{fig: mean shift}
\end{figure}

We begin with three controlled experiments illustrating how KL divergence and OT distances respond to support mismatch, posterior concentration and displacement of the posterior mean.

Figure~\ref{fig: support missmatch} visualizes the main advantage of OT distances compared to the KL divergence. OT distances do not require absolute continuity and can therefore remain finite even for posteriors that are not absolutely contious w.r.t.\ the prior. Specifically, the prior is supported on $[-2,2]$ and the posterior is defined as $\posterior_\kappa=(1-\kappa)\posterior_{\mathrm{in}}+\kappa\posterior_{\mathrm{out}}$, where $\posterior_{\mathrm{in}}$ is supported on $[-2,2]$ and $\posterior_{\mathrm{out}}$ is supported outside this interval. Thus, $\kappa$ controls the fraction of posterior mass lying outside the support of the prior. At $\kappa=0$, the posterior is absolutely continuous w.r.t.\ the prior and the KL divergence is finite. For every $\kappa>0$, absolute continuity is violated and the KL divergence becomes infinite. In contrast, the Wasserstein distance remains finite and increases smoothly with $\kappa$, reflecting the amount and distance of transported mass.

Wasserstein distances can also remain well behaved when the posterior and prior have common support and the posterior is absolutely continuous w.r.t.\ to the prior. In Figure~\ref{fig: gaussian collapse}, we consider $\prior=\calN(0,\sigma_\prior^2 \idmat_d)$ and $\posterior_\sigma=\calN(\mu,\sigma^2 \idmat_d)$, both of which have support on $\R^d$ for $\sigma>0$ and the posterior is absolutely continuous w.r.t.\ the prior. We progressively decrease $\sigma$, causing the posterior to concentrate around $\mu$. The behavior visualized in the figure can be obtained directly from the closed-form expressions. In particular,
\begin{align*}
    \KL{\posterior_\sigma}{\prior} = d\ln\frac{\sigma_\prior}{\sigma} + \frac{d\sigma^2+\|\mu\|_2^2}{2\sigma_\prior^2} - \frac{d}{2},
\end{align*}
so, as $\sigma \to 0$, KL is dominated by the logarithmic term $d\ln(\sigma_\prior/\sigma)$ and therefore blows up. In contrast, the Bures-Wasserstein formula gives
\begin{align*}
    W_2^2(\posterior_\sigma,\prior) = \|\mu\|_2^2+d(\sigma_\prior-\sigma)^2,
\end{align*}
and hence
\begin{align*}
    W_2^2(\posterior_\sigma,\prior) \longrightarrow \|\mu\|_2^2+d\sigma_\prior^2 \qquad\text{as }\sigma\to0.
\end{align*}
More generally, for any $p\geq1$, the Wasserstein distance can be bounded by
\begin{align*}
    	W_p^p(\posterior_\sigma,\prior) \leq \Expval\left[ \left\|\mu+(\sigma_\prior-\sigma)\RVZ\right\|^p \right],
\end{align*}
where $\RVZ \sim \calN(0,\idmat_d)$. Since the right-hand side converges to
\begin{align*}
    \Expval\left[\left\|\mu+\sigma_\prior Z\right\|^p\right] < \infty
\end{align*}
as $\sigma\to0$, so that the Wasserstein distance remains bounded for any $p \geq 1$. Thus, while the KL divergence diverges as the posterior collapses to the deterministic distribution $\delta_\mu$, the Wasserstein distances of every order $p\geq1$ remain bounded and retain a direct geometric interpretation. 
The plots in Figure~\ref{fig: gaussian collapse} reflect precisely this behavior.

Finally, Figure~\ref{fig: mean shift} examines the effect of translating the posterior while keeping its covariance fixed. We set $\prior = \calN(0,\sigma^2 \idmat_d)$ and $\posterior_\mu=\calN(\mu,\sigma^2 \idmat_d)$, so that the posterior remains absolutely continuous w.r.t.\ the prior. In this case, the general KL expression simplifies to
\begin{align*}
    \KL{\posterior_\mu}{\prior} = \frac{\|\mu\|_2^2}{2\sigma^2},
\end{align*}
whereas the Bures-Wasserstein formula yields
\begin{align*}
    W_2(\posterior_\mu,\prior) = \|\mu\|_2.
\end{align*}
Consequently, both discrepancies increase with the magnitude of the translation, but at different rates. The KL divergence grows quadratically in $|\mu|_2$, while the Wasserstein distance grows linearly. More generally, for any $p \geq 1$, the equal-scale Gaussian case also admits the exact expression
\begin{align*}
    W_p(\posterior_\mu,\prior)=\|\mu\|_2.
\end{align*}

Thus, the qualitative behavior shown in Figure~\ref{fig: mean shift} follows directly from the corresponding closed-form expressions.

Together, these experiments demonstrate that OT distances respond continuously to geometric changes in the posterior. By contrast, KL divergence becomes singular under support mismatch or posterior collapse and grows quadratically under mean shifts. These differences provide the motivation for replacing KL-based complexity terms with geometry-aware discrepancies in the PAC-Bayesian bounds studied below.

\section{Adult Income Classification}
\begin{figure}[htbp]
    \centering
    \begin{subfigure}[b]{0.48\textwidth}
        \centering
        \includegraphics[width=\textwidth]{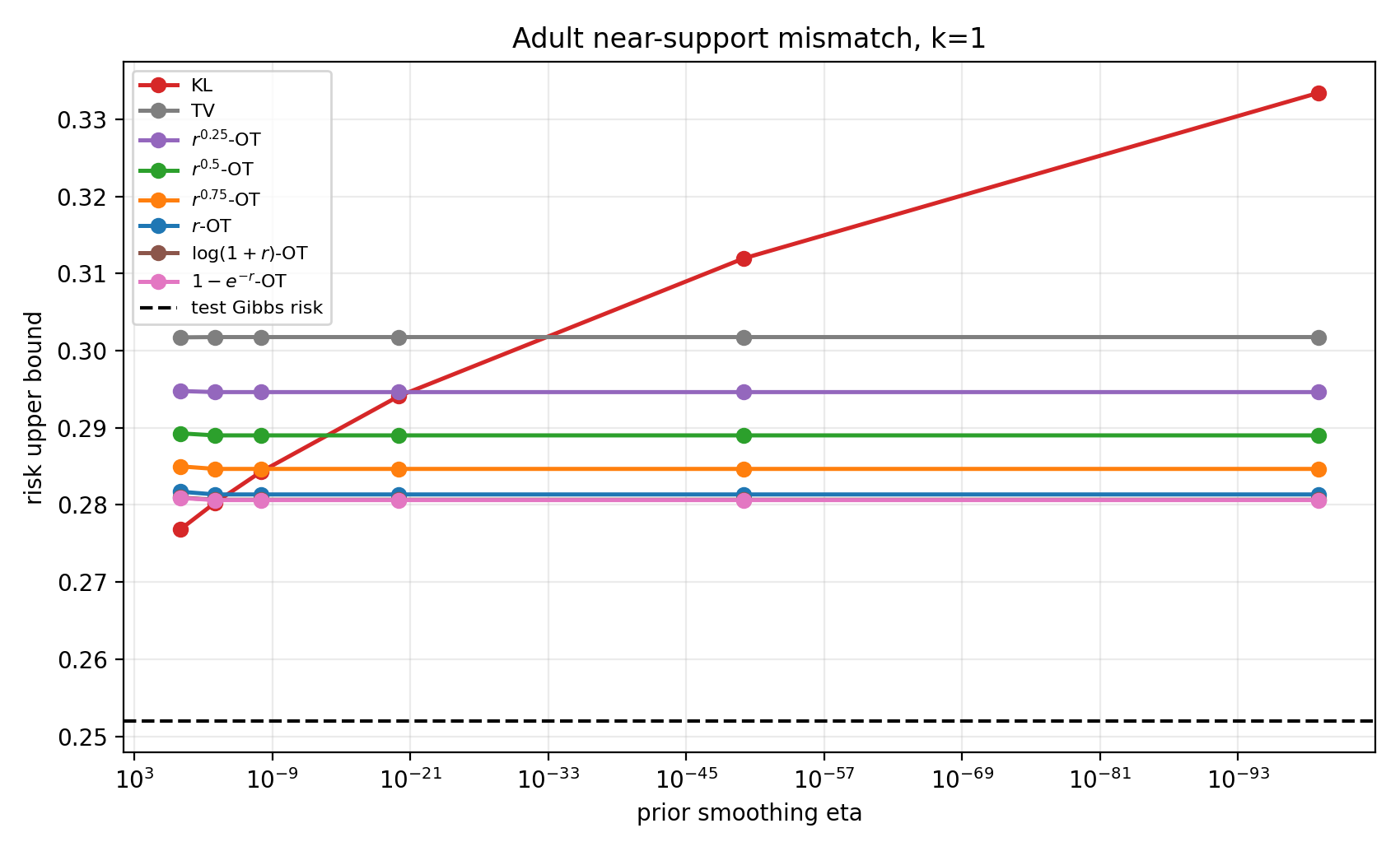} 
        \caption{$\eta$ sweep}
        \label{fig: adult data set eta sweep}
    \end{subfigure}
    \hfill
    \begin{subfigure}[b]{0.48\textwidth}
        \centering
        \includegraphics[width=\textwidth]{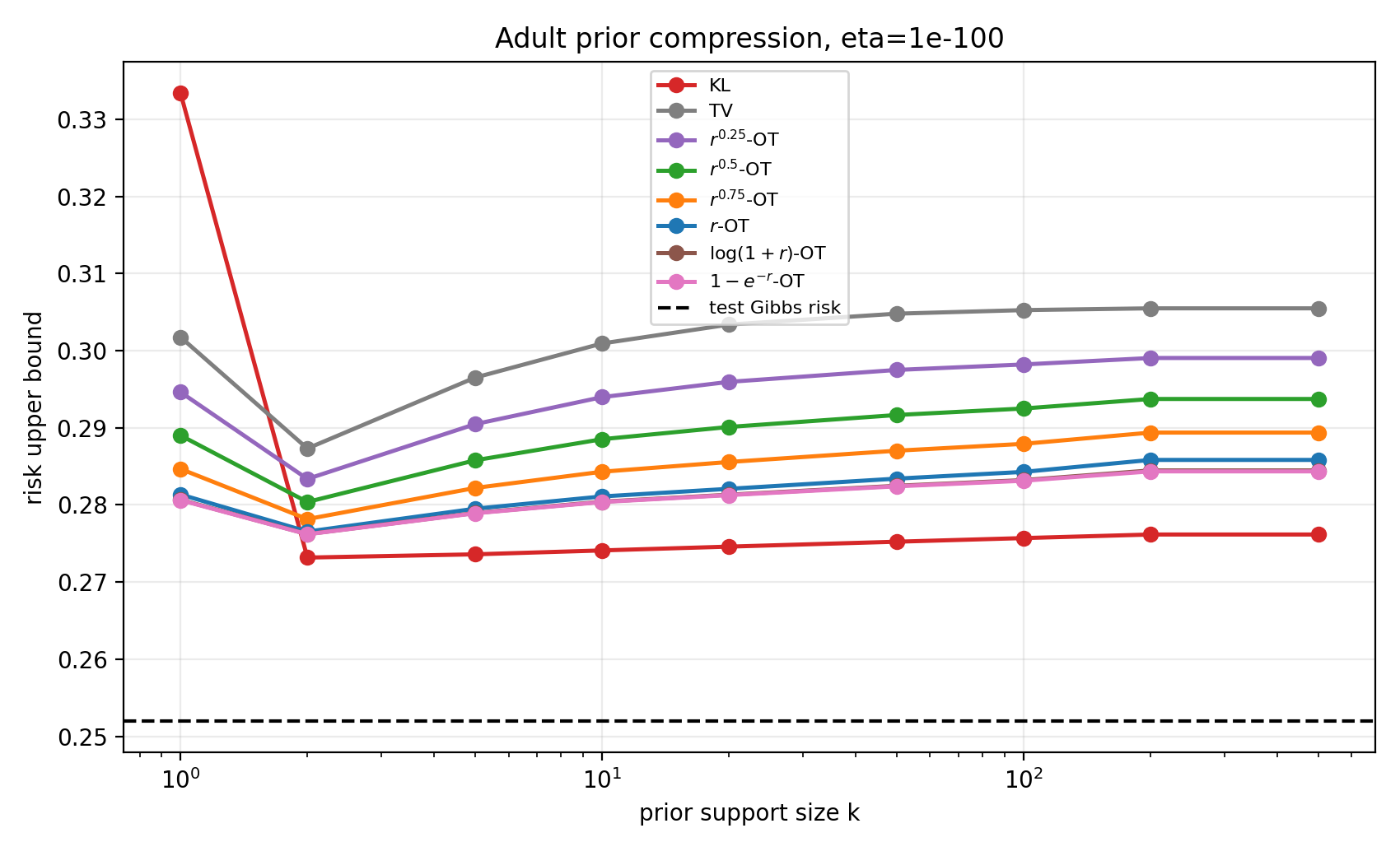}
        \caption{k sweep}
        \label{fig: adult data set k sweep}
    \end{subfigure}

    \caption{In this experiment, we compare PAC-Bayesian bounds whose complexity terms are based on different statistical distance. We compare the behavior of the bounds for different priors and posteriors.}
    \label{fig: adult income}
\end{figure}

The second experiment evaluates the PAC Bayesian bounds on a real-world tabular classification task. We use the UCI Adult dataset to predict whether an individual's income exceeds $50K$. After removing observations with missing values, we one-hot encode categorical features, standardize numerical features, and normalize each feature vector to have Euclidean norm at most one. We stratify the data into a model-zoo split ($S_{\mathrm{zoo}}$), a bound split ($S_{\mathrm{bound}}$), and a test split ($S_{\mathrm{test}}$), with proportions $(35\%/35\%/30\%)$.

We construct a finite hypothesis class by training a randomized zoo of $200$ linear classifiers on $S_{\mathrm{zoo}}$. We introduce diversity through bootstrap subsampling, random feature masks, varying regularization strengths, and varying positive-class weights. For each classifier, we use the clipped margin loss
\begin{align*}
    \loss(h,(x,y)) = \min\{1,\max\{0,1-yh(x)\}\},
\end{align*}
where $y\in\{-1,1\}$. The posterior $\posterior(h)\propto \exp \left(-\beta \emprisk(h, S_\mathrm{bound})\right)$ is a Gibbs distribution over the model zoo. We evaluate its empirical risk on $S_{\mathrm{bound}}$ and estimate its true risk on $S_{\mathrm{test}}$.

We design the prior to examine prior compression and near-support mismatch. Using an independent subset of $S_{\mathrm{zoo}}$, we rank the hypotheses by their risk and concentrate the prior on the top $k$ classifiers. A smoothing parameter $\eta$ mixes this top-$k$ distribution with the uniform distribution over the complete model zoo. We sweep
\begin{align*}
    k \in \{1,2,5,10,20,50,100,200\} \quad \text{and} \quad \eta \in \{10^{-1},10^{-4},10^{-8},10^{-20},10^{-50},10^{-100},0\}.
\end{align*}
Small values of $\eta$ therefore yield a nearly support-deficient prior. Consequently, the KL divergence becomes large and when $\eta=0$ and $k<200$, is infinite because $\posterior$ assigns positive mass outside the support of $\prior$.

For each prior, we compute the KL divergence, and several concave OT discrepancies on the finite hypothesis class. The ground metric is the clipped and rescaled Euclidean distance between the parameter vectors of the linear classifiers. We consider the transformations $\phi(r)=r^\alpha$ for $\alpha\in{0.25,0.5,0.75,1}$, $\phi(r)=\ln(1+r)$, and $\phi(r)=1-\exp(-r)$. This experiment complements the synthetic regression experiment by examining support mismatch in a finite, realistic hypothesis class, with a data-dependent prior constructed independently of the bound sample.

\end{document}